\pdfoutput=1
\documentclass{article}

\usepackage[preprint]{neurips_2026}

\newif\ifincludeneuripschecklist
\includeneuripschecklistfalse

\usepackage[utf8]{inputenc}
\usepackage[T1]{fontenc}
\usepackage{graphicx}
\usepackage{amsmath, amssymb, amsthm}
\usepackage{mathtools}
\usepackage{bm}
\usepackage{booktabs}
\usepackage{hyperref}
\usepackage{cleveref}
\usepackage{xcolor}
\usepackage{url}

\theoremstyle{plain}
\newtheorem{theorem}{Theorem}
\newtheorem{proposition}[theorem]{Proposition}
\newtheorem{lemma}[theorem]{Lemma}
\newtheorem{corollary}[theorem]{Corollary}

\theoremstyle{definition}
\newtheorem{definition}[theorem]{Definition}

\theoremstyle{remark}
\newtheorem{remark}[theorem]{Remark}

\newcommand{\R}{\mathbb{R}}
\newcommand{\E}{\mathbb{E}}
\newcommand{\Prob}{\mathbb{P}}
\newcommand{\Agents}{\mathcal{A}}        % agent set
\newcommand{\Normal}{\mathcal{N}}        % Gaussian distribution
\newcommand{\Obs}{\mathcal{O}}           % observation space
\newcommand{\Rollouts}{\mathcal{B}}      % rollout observation set
\newcommand{\Pairs}{\mathcal{P}}         % set of unordered pairs
\newcommand{\policy}[1]{\pi_{\theta_{#1}}}
\newcommand{\Wass}{W_2}
\newcommand{\SND}{\ensuremath{\mathrm{SND}}}
\newcommand{\SNDG}{\mathrm{SND}_{G}}

\newcommand{\abs}[1]{\left\lvert #1 \right\rvert}

\newcommand{\ind}{\mathbf{1}}

\newcommand{\codelocation}{Code: \url{https://github.com/shawnray-research/Graph-SND}.}

\title{Graph-SND: Sparse Aggregation for Behavioral Diversity in
Multi-Agent Reinforcement Learning}

\author{%
  Shawn Ray \\
  Carnegie Mellon University \\
  \texttt{shawnray@cmu.edu} \\
}

\begin{document}
\maketitle

\begin{abstract}
System Neural Diversity (SND) measures behavioral heterogeneity in
multi-agent reinforcement learning by averaging pairwise distances over
all $\binom{n}{2}$ agent pairs, making each call quadratic in team size.
We introduce Graph-SND, which replaces this complete-graph average with
a weighted average over the edges of an arbitrary graph $G$. Three
regimes follow: $G=K_n$ recovers SND exactly; a fixed sparse $G$ defines
a localized diversity measure at $O(|E|)$ cost; and random edge samples
yield an unbiased Horvitz-Thompson estimator and a normalized sample
mean with $O(1/\sqrt{m})$ concentration in the sampled edge count $m$.
For fixed sparse graphs we prove forwarding-index distortion bounds for
expanders and a spectral refinement under low-rank distance structure;
for random $d$-regular graphs we prove an unconditional probabilistic
$\widetilde{\mathcal{O}}(D_{\max}/\sqrt{n})$ bound. On VMAS we verify
recovery, unbiasedness, concentration, and wall-clock scaling, with a
PettingZoo TVD panel checking non-Gaussian transfer. In a 500-iteration
$n=100$ PPO run, Bernoulli-$0.1$ Graph-SND tracks full SND while
reducing per-call metric time by about $10\times$, and frozen-policy
GPU timing up to $n=500$ follows the predicted $\binom{n}{2}/|E|$
speedup. Random $d$-regular expanders empirically achieve
$\SNDG^{\mathrm{u}}/\SND\in[0.9987,1.0013]$ at $\Theta(n\log n)$ edges.
In DiCo diversity control at $n=50$, Bernoulli-$0.1$ Graph-SND preserves
set-point tracking with paired reward differences indistinguishable from
zero across nine matched cells while cutting per-call metric cost by
${\sim}9.5\times$. Together, these results show that the SND aggregation
bottleneck can be removed without changing the metric's semantics,
yielding a drop-in sparse alternative that scales beyond complete-graph
SND and supports both passive measurement and closed-loop diversity
control.
\end{abstract}
% ============================================================
\section{Introduction}
\label{sec:intro}

Behavioral diversity helps multi-agent reinforcement learning (MARL)
systems specialize, remain robust, and expose capabilities not visible in
reward alone \citep{bettini2023hetgppo,bettini2025snd}. System Neural
Diversity (SND) \citep{bettini2025snd} measures this heterogeneity by
averaging pairwise Wasserstein distances between all agents' action
distributions and enables diversity control through DiCo
\citep{bettini2024dico}. Its bottleneck is aggregation: every call uses
all $\binom{n}{2}$ agent pairs, and \citet[\S6]{bettini2025snd} identify
subgroup-to-system aggregation as open.

We introduce Graph-SND, which replaces SND's complete-graph average by a
weighted average over the edges of a graph $G$. The complete graph recovers
SND exactly. A fixed graph, such as a communication or nearest-neighbor
graph, defines a localized diversity measure over task-relevant
interactions. A random graph defines sparse estimators of full SND,
including an unconditionally unbiased Horvitz-Thompson estimator and a
normalized sample mean with finite-population concentration. This isolates
the aggregation bottleneck while leaving the policy architecture, training
algorithm, and pairwise distance unchanged.

\paragraph{Contributions.}
(i) We define Graph-SND (Definition~\ref{def:graph-snd}), a sparse
aggregation layer that recovers SND on $K_n$ and costs $O(|E|)$ on a
precomputed graph. (ii) We prove recovery, non-negativity,
graph-automorphism invariance, Horvitz-Thompson unbiasedness,
conditional concentration, forwarding-index distortion bounds
\citep{heydemann1989forwarding,leighton1999multicommodity}, a
spectral nuclear-norm refinement, and an unconditional probabilistic
$\widetilde{\mathcal{O}}(D_{\max}/\sqrt{n})$ bound for random
$d$-regular graphs. (iii) VMAS experiments
\citep{bettini2022vmas} verify recovery, concentration, and wall-clock
scaling through $n=500$, while PettingZoo provides a discrete-action
TVD transfer check; a 500-iteration $n=100$ PPO run
shows Bernoulli-$0.1$ Graph-SND tracking full SND at about $10\times$
lower metric cost. (iv) Random $d$-regular expanders achieve
$\SNDG^{\mathrm{u}}/\SND\in[0.9987,1.0013]$ at
$\Theta(n\log n)$ edges, and DiCo at $n=50$ shows Bernoulli-$0.1$
Graph-SND preserves set-point tracking and yields paired reward differences indistinguishable from zero across nine matched cells while reducing per-call metric cost by ${\sim}9.5\times$.

% ============================================================
\section{Related Work}
\label{sec:related}

\paragraph{Diversity measures in multi-agent RL.}
Prior metrics compare policy distributions
\citep{mckee2022diversity,yu2021informative,hu2022policy}, occupancy
measures \citep{liu2021unifying}, deterministic action diversity
\citep{parker-holder2020effective}, or trajectory distributions
\citep{lupu2021trajedi}. SND \citep{bettini2025snd} is closest: it is a
metric on concurrently acting Gaussian policies and is used by DiCo
\citep{bettini2024dico}, measurement studies
\citep{bettini2024impact,tessera2024hypermarl}, and diversity-benefit
analysis \citep{amir2025diversity}. These works treat the all-pairs
average as fixed. Graph-SND changes only this aggregation layer.

\paragraph{Heterogeneity in robotics.}
Robotics heterogeneity is often structural rather than behavioral:
species-trait matrices \citep{prorok2017impact}, performance proxies
\citep{li2004learning}, and fixed-class counts \citep{twu2014measure}.
Hierarchic Social Entropy \citep{balch2000hse} is closer, but
\citet{bettini2025snd} show it does not capture redundancy in the same
way as SND. Graph-SND inherits SND's complete-graph behavior and adds a
structured sparse alternative.

\paragraph{Graphs in multi-agent learning, and subgroup partitioning.}
Graphs in MARL usually appear inside policies as communication or
message-passing structure
\citep{sukhbaatar2016learning,blumenkamp2021emergence,bettini2023hetgppo}
or as role/credit inductive bias \citep{wang2020roma,wang2021rode}. Our
graph is external: it selects which pairwise behavioral distances enter
the scalar diversity aggregate. Subgroup partitions proposed by
\citet[\S6]{bettini2025snd} are a special case (block-diagonal unions of
cliques), while $k$-NN, communication, Bernoulli, and expander graphs are
not representable as simple partitions.

\paragraph{Sampling-based approximations, and follow-up work on SND.}
The random-$G$ view uses Horvitz-Thompson estimation
\citep{horvitz1952generalization}, Hoeffding's finite-population
concentration \citep{hoeffding1948class,hoeffding1963probability}, and
Serfling-style refinements \citep{serfling1974probability,bardenet2015concentration}.
The fixed-expander view connects MARL diversity aggregation to edge
forwarding index and multicommodity-flow structure
\citep{chung1987forwarding,heydemann1989forwarding,leighton1999multicommodity},
which prior pairwise diversity metrics do not exploit.

% ============================================================
\section{Background}
\label{sec:background}

\subsection{Partially Observable Markov Games}
\label{sec:background:pomg}

We work in cooperative Partially Observable Markov Games
\citep{kaelbling1998planning} with agents
$\Agents=\{1,\ldots,n\}$, private observations, shared dynamics, and
heterogeneous stochastic policies $\policy{i}$. Agents share observation
and action spaces so their action distributions can be compared at the
same observation; heterogeneous spaces would require an alignment map and
are outside this paper's scope. The diversity metrics are computed on
rollout data from independently parameterized policies, as in
\citet{bettini2023hetgppo,bettini2025snd}.

\subsection{System Neural Diversity}
Following \citet{bettini2025snd}, let $\Agents = \{1, \dots, n\}$ be the
agent set and let each agent $i \in \Agents$ have a Gaussian policy
$\policy{i}(o) = \Normal\bigl(\mu_{\theta_i}(o),\, \Sigma_{\theta_i}(o)\bigr)$.
Let $\Rollouts = \{o^t\}_{t=1}^T$ be a collection of joint observation
vectors sampled from policy rollouts, where each
$o^t = (o_1^t,\dots,o_n^t) \in \Obs^n$.
The Monte Carlo behavioral distance between agents $i$ and $j$ is
\begin{equation}
\label{eq:dij}
d(i,j) \;=\; \frac{1}{|\Rollouts|\,|\Agents|}
\sum_{o^t \in \Rollouts} \sum_{k \in \Agents}
\Wass\bigl(\policy{i}(o^t_k),\, \policy{j}(o^t_k)\bigr),
\end{equation}
where $\Wass$ admits a closed form on multivariate Gaussians
\citep{bettini2025snd}. The global System Neural Diversity is
\begin{equation}
\label{eq:snd}
\SND(\mathbf{D}) \;=\; \binom{n}{2}^{-1} \sum_{i<j} d(i,j),
\end{equation}
the mean of pairwise behavioral distances over unique agent pairs. We work
with~\eqref{eq:dij} directly rather than its integral form, avoiding
specification of a measure on $\Obs$; all our results are stated at the
level of the sampled estimator. Throughout, we treat the values
$\{d(i,j)\}_{\{i,j\} \in \Pairs}$ as fixed (deterministic in $\mathbf{D}$),
so that all stochasticity in the estimators below arises solely from the
random graph.

% ============================================================
\section{Graph-SND}
\label{sec:method}

Full SND weights every pair equally. This is both computationally
prohibitive (quadratic in $n$) and, in many settings, semantically
imprecise: agents that do not interact or operate on disjoint sub-tasks
contribute to the score on equal footing with coupled pairs. We introduce a
graph-structured generalization that supports two distinct use cases: a
localized diversity measure when $G$ encodes interaction structure, and
sampling-based estimators of $\SND$ when edges of $G$ are drawn from the
pair set.

\subsection{Definition}

Let $\Pairs := \binom{\Agents}{2}$ be the set of unordered agent pairs,
and let $G = (\Agents, E, w)$ be an undirected graph on the agent set with
$E \subseteq \Pairs$ and non-negative edge weights $w : E \to \R_{\geq 0}$.
Write $W(G) = \sum_{\{i,j\} \in E} w_{ij}$ and assume $W(G) > 0$. The graph
may be fixed (e.g., communication topology, $k$-NN on observations,
$\varepsilon$-ball) or sampled from pairs
(Section~\ref{sec:method:estimator}).

\begin{definition}[Graph-SND]
\label{def:graph-snd}
The Graph-SND of policies $\{\policy{i}\}_{i \in \Agents}$ on $G$ is
\begin{equation}
\label{eq:graph-snd}
\SNDG(\mathbf{D}, G) \;=\; \frac{1}{W(G)}
\sum_{\{i,j\} \in E} w_{ij} \, d(i,j),
\end{equation}
with $d(i,j)$ given by~\eqref{eq:dij}.
\end{definition}

When $G = K_n$ with unit weights, $W(G) = \binom{n}{2}$ and
$\SNDG(\mathbf{D}, K_n) = \SND(\mathbf{D})$ identically.

\subsection{Canonical graph families}

Three choices of $G$ recover interpretable special cases.
\begin{itemize}
  \item \textbf{Complete graph}, $G = K_n$, $w \equiv 1$: recovers $\SND$.
  \item \textbf{$k$-nearest-neighbor graph} in observation space:
  $|E| = O(nk)$, yielding a $\Theta(n/k)$ reduction in the number of
  pairwise distance evaluations relative to full $\SND$.
  \item \textbf{Communication graph} inherited from the policy's message-passing topology: aligns measured diversity with the interaction
  structure actually used by the agents.
\end{itemize}
In each case $G$ may be fixed or state-dependent; the latter requires
re-evaluating $\SNDG$ per rollout batch.

For any weighted graph $G' = (\Agents', E', w')$ with $W(G') = 0$, we define
\[
\SNDG(\mathbf{D}, G') := 0
\]
by convention.

% ============================================================
\section{Theoretical Analysis}
\label{sec:theory}

We state Graph-SND's core properties, then give a probabilistic
interpretation as sampling-based estimators of $\SND$ with concentration
guarantees. Proofs are in Appendix~\ref{app:proofs}.

\subsection{Core properties}

\begin{proposition}[Recovery on the complete graph]
\label{prop:recovery}
If $G = K_n$ with unit edge weights, then
$\SNDG(\mathbf{D}, G) = \SND(\mathbf{D})$.
\end{proposition}

\begin{proposition}[Non-negativity]
\label{prop:nonneg}
$\SNDG(\mathbf{D}, G) \geq 0$, with equality if and only if $d(i,j) = 0$
for every edge $\{i,j\} \in E$ with $w_{ij} > 0$.
\end{proposition}

\begin{remark}[Semantic shift from $\SND$]
\label{rem:semantics}
Proposition~\ref{prop:nonneg} is strictly weaker than the zero condition
of full $\SND$: $\SNDG$ can vanish even when the system contains
behaviorally distinct agents, provided those agents are not adjacent in
$G$. When $G$ encodes task-relevant interaction structure, this is the
intended semantics: Graph-SND measures \emph{locally relevant}
diversity. When full-system diversity is the target, the estimator
interpretation of Section~\ref{sec:method:estimator} is the appropriate
one.
\end{remark}

\begin{proposition}[Graph-automorphism invariance]
\label{prop:perm}
For any automorphism $\sigma$ of $(G, w)$,
$\SNDG(\sigma \cdot \mathbf{D}, G) = \SNDG(\mathbf{D}, G)$, where
$(\sigma \cdot \mathbf{D})_{ij} = d(\sigma(i), \sigma(j))$.
\end{proposition}

\begin{proposition}[Complexity]
\label{prop:complexity}
Evaluating $\SNDG$ on a precomputed graph $G$ requires $|E|$ pairwise
distance computations $d(\cdot,\cdot)$, versus $\binom{n}{2}$ for $\SND$.
For $k$-NN graphs with $|E| = O(nk)$, this yields a speedup factor of
$\Theta(n/k)$ in the number of pairwise distance evaluations. This
statement excludes the cost of constructing $G$; measured construction
overheads for the graph builders used here appear in
Appendix~\ref{app:verify} (Table~\ref{tab:graph-build}).
\end{proposition}

\subsection{Sampling estimators of $\SND$}
\label{sec:method:estimator}

When the goal is to approximate full $\SND$ at reduced cost, we derive
sampling-based estimators from the same graph formalism. Let $G_p$ include
each pair $\{i,j\} \in \Pairs$ independently with probability
$p \in (0, 1]$. We distinguish two random-graph statistics. The
Horvitz-Thompson statistic uses the population normalization
$|\Pairs|^{-1}$ and is unbiased even when the sampled edge set is empty
(the empty sum contributes zero). The normalized Graph-SND statistic uses
the realized sample size $m=|E(G_p)|$ and is the sample mean over sampled
edges; it is defined conditional on $m\geq1$ and is the quantity used in
the concentration theorem below.

\begin{proposition}[Unbiasedness]
\label{prop:unbiased}
With weights $w_{ij}=1/p$, the Horvitz-Thompson estimator
\[
\widehat{\SND}_{\mathrm{HT}}(G_p) \;:=\; |\Pairs|^{-1}
\sum_{\{i,j\} \in E(G_p)} w_{ij}\, d(i,j)
\]
satisfies
$\E_{G_p}\bigl[\widehat{\SND}_{\mathrm{HT}}(G_p)\bigr]
= \SND(\mathbf{D})$.
\end{proposition}

\begin{remark}[Normalized Graph-SND is a sample mean]
\label{rem:ht-sample-mean}
If the same $1/p$ weights are inserted into the normalized Graph-SND
definition, $W(G_p) = |E(G_p)|/p$, so the weights cancel and $\SNDG$
reduces to the uniform sample mean:
\[
\SNDG(\mathbf{D}, G_p)
\;=\; \frac{p}{|E(G_p)|} \sum_{\{i,j\} \in E(G_p)} \frac{1}{p}\, d(i,j)
\;=\; \frac{1}{|E(G_p)|} \sum_{\{i,j\} \in E(G_p)} d(i,j).
\]
This \emph{uniform-weight variant} differs from
$\widehat{\SND}_{\mathrm{HT}}$ in its random normalization. It is used in
Theorem~\ref{thm:concentration} conditional on $|E(G_p)|\geq1$. In the
DiCo Bernoulli implementation, an empty draw falls back to the full pair
set for that call so the control ratio remains well-defined.
\end{remark}

\begin{theorem}[Concentration]
\label{thm:concentration}
Assume $d(i,j) \in [0, D_{\max}]$ for all $\{i,j\} \in \Pairs$. Let
$m = |E(G_p)|$, and define the uniform-weight variant
\[
\SNDG^{\mathrm{u}}(\mathbf{D}, G_p) \;:=\; m^{-1}
\sum_{\{i,j\} \in E(G_p)} d(i,j).
\]
Then for any $\delta \in (0,1)$, conditional on $|E(G_p)| = m \geq 1$,
with probability at least $1 - \delta$ over the randomness of $G_p$,
Hoeffding's sampling-without-replacement finite-population extension gives
\begin{equation}
\label{eq:concentration}
\abs{\SNDG^{\mathrm{u}}(\mathbf{D}, G_p) - \SND(\mathbf{D})}
\;\leq\; D_{\max} \sqrt{\frac{\log(2/\delta)}{2m}}.
\end{equation}
\end{theorem}

\begin{remark}[Unconditional Bernoulli form]
\label{rem:chernoff-m}
Let $\mu=p|\Pairs|$ be the expected number of sampled edges in $G_p$.
Combining Theorem~\ref{thm:concentration} with the Chernoff bound
$\Prob(|E(G_p)|<\mu/2)\leq e^{-\mu/8}$ gives, for any
$\delta\in(0,1)$, with probability at least
$1-\delta-e^{-\mu/8}$,
\[
\abs{\SNDG^{\mathrm{u}}(\mathbf{D},G_p)-\SND(\mathbf{D})}
\;\leq\; D_{\max}\sqrt{\frac{\log(2/\delta)}{\mu}},
\]
on the event $|E(G_p)|\geq1$. Thus Bernoulli sampling has the same
finite-population $O(1/\sqrt{p|\Pairs|})$ rate up to constants whenever
the expected edge count is not tiny.
\end{remark}

\begin{remark}[Sharper finite-population bounds]
\label{rem:serfling}
For sampling fractions $m/|\Pairs|$ bounded away from zero, Serfling's
inequality \citep{serfling1974probability} gives a bound of the form
$D_{\max}\sqrt{(1 - (m-1)/|\Pairs|)\log(2/\delta)/(2m)}$, which is strictly
sharper. See also \citet{bardenet2015concentration} for further
refinements and an empirical-Bernstein variant.
\end{remark}

Theorem~\ref{thm:concentration} gives a conditional rate of
$O(1/\sqrt{m})$ in the number of sampled edges. For target error
$\varepsilon$, it suffices to sample
$m = O\bigl(D_{\max}^2\, \varepsilon^{-2}\, \log(1/\delta)\bigr)$ edges,
which for large $n$ is vastly smaller than $\binom{n}{2}$.

For fixed sparse graphs, let $\pi(G)$ denote the edge forwarding index:
the minimum over all pairwise path routings of the maximum number of
routed paths using any edge~\citep{chung1987forwarding,heydemann1989forwarding}.

\begin{theorem}[Deterministic fixed-$G$ distortion]
\label{thm:distortion}
Let $G=(\Agents,E)$ be connected with unit edge weights, and suppose
$d:\Pairs\to\R_{\ge0}$ satisfies the triangle inequality. With
$\SNDG^{\mathrm{u}}(\mathbf{D},G):=|E|^{-1}\sum_{\{i,j\}\in E}d(i,j)$,
\begin{equation}
\label{eq:distortion}
\frac{|E|}{|\Pairs|}\SNDG^{\mathrm{u}}(\mathbf{D},G)
\leq \SND(\mathbf{D})
\leq
\frac{|E|\pi(G)}{|\Pairs|}\SNDG^{\mathrm{u}}(\mathbf{D},G).
\end{equation}
Thus $|E|/|\Pairs|\leq\SND/\SNDG^{\mathrm{u}}\leq|E|\pi(G)/|\Pairs|$
whenever $\SNDG^{\mathrm{u}}>0$, and both ratios are $1$ on $K_n$.
\end{theorem}

\begin{corollary}[Expander sparsification]
\label{cor:expander-distortion}
For a $d$-regular spectral expander,
$\pi(G)=\mathcal{O}(n\log n/d)$~\citep{leighton1999multicommodity}, so
Theorem~\ref{thm:distortion} gives $\mathcal{O}(\log n)$ worst-case
relative distortion using only $|E|=\Theta(n\log n)$ pairwise distances
when $d=\Theta(\log n)$.
\end{corollary}

\begin{remark}[Spectral sharpening]
\label{rem:main-spectral-sharpening}
Proposition~\ref{prop:spectral-nuclear-discrepancy} (Appendix~\ref{app:proofs})
sharpens this fixed-graph guarantee under controlled normalized nuclear
norm $\rho_*(\mathbf{D})=\|\mathbf{D}\|_*/(n\,\SND(\mathbf{D}))$: for
$d$-regular expanders, the relative error scales as
$\mathcal{O}(\rho_*(\mathbf{D})/\sqrt d)$, hence
$\mathcal{O}(\rho_*/\sqrt{\log n})$ at $d=\Theta(\log n)$. This does not
prove the empirical near-unit ratios, but it explains why low-effective-rank
MARL distance matrices can be much easier than the worst case.
\end{remark}

\begin{remark}[Unconditional probabilistic bound]
\label{rem:main-probabilistic-bound}
Proposition~\ref{prop:probabilistic-distortion} (Appendix~\ref{app:proofs})
gives a fully unconditional concentration bound: for $G$ drawn uniformly
from simple $d$-regular graphs, with probability at least $1-\delta$,
$|\SNDG^{\mathrm{u}}-\SND|\leq\widetilde{\mathcal{O}}(D_{\max}/\sqrt{n})$
at $d=\Theta(\log n)$, requiring only $D_{ij}\in[0,D_{\max}]$. This
provides a distribution-free probabilistic complement to the worst-case
$\mathcal{O}(\log n)$ ratio of
Corollary~\ref{cor:expander-distortion}, though sharper
distribution-dependent analysis would be needed to fully explain the
empirically observed near-unit ratios in
Section~\ref{sec:experiments:expander}.
\end{remark}

Both Theorems~\ref{thm:concentration} and~\ref{thm:distortion} depend
only on boundedness (and for Theorem~\ref{thm:distortion}, triangle
inequality) of $d(i,j)$, not on the Gaussian/Wasserstein instantiation;
discrete action spaces may substitute bounded metrics such as total
variation or the square-root Jensen-Shannon metric.

% ============================================================
\section{Experiments}
\label{sec:experiments}

We evaluate Graph-SND as an aggregation layer, so the experiments target
metric recovery, estimator scaling, deterministic sparse approximation,
and downstream DiCo control rather than new task-level RL performance.
Hyperparameters and rollout details are in Appendix~\ref{app:hyperparams}.
Core checks on VMAS Multi-Agent Goal Navigation verify exact recovery on $K_4$,
Horvitz-Thompson unbiasedness at $n=8$, and zero concentration-bound
violations across $12$ cells with $2{,}000$ finite-population draws each
at $n=16$
(Appendix~\ref{app:verify}).

\subsection{Scaling to $n=100$ under training dynamics}
\label{sec:experiments:scaling}

We train $100$ independently parameterized Gaussian policies on VMAS
Multi-Agent Goal Navigation for $500$ PPO iterations. Every $25$ iterations we log
full $\SND$ over all $\binom{100}{2}=4{,}950$ pairs and
$\SNDG^{\mathrm{u}}(G_{0.1})$ on one Bernoulli-$0.1$ edge sample, with
CUDA-synchronized timing. Figure~\ref{fig:scaling-n100} shows the two
diversity curves overlapping throughout training while the sampled call
stays near the predicted $1/p=10\times$ speedup
($[9.0,10.9]\times$ over all logged measurements). A separate frozen-init
GPU timing sweep through $n=500$ follows the same
$\binom{n}{2}/|E|$ law: at $n=500$, Bernoulli-$0.1$ reduces mean metric
time from $15{,}502.8$ ms to $1{,}553.6$ ms
(Figure~\ref{fig:timing-n500}). Together, the online $n=100$ run and
the frozen $n=500$ sweep separate the two questions a reviewer should
care about: the sampled metric tracks full SND during learning, and the
wall-clock reduction persists at the larger team sizes where the
quadratic complete graph becomes expensive.

\begin{figure}[t]
\centering
\includegraphics[width=\linewidth]{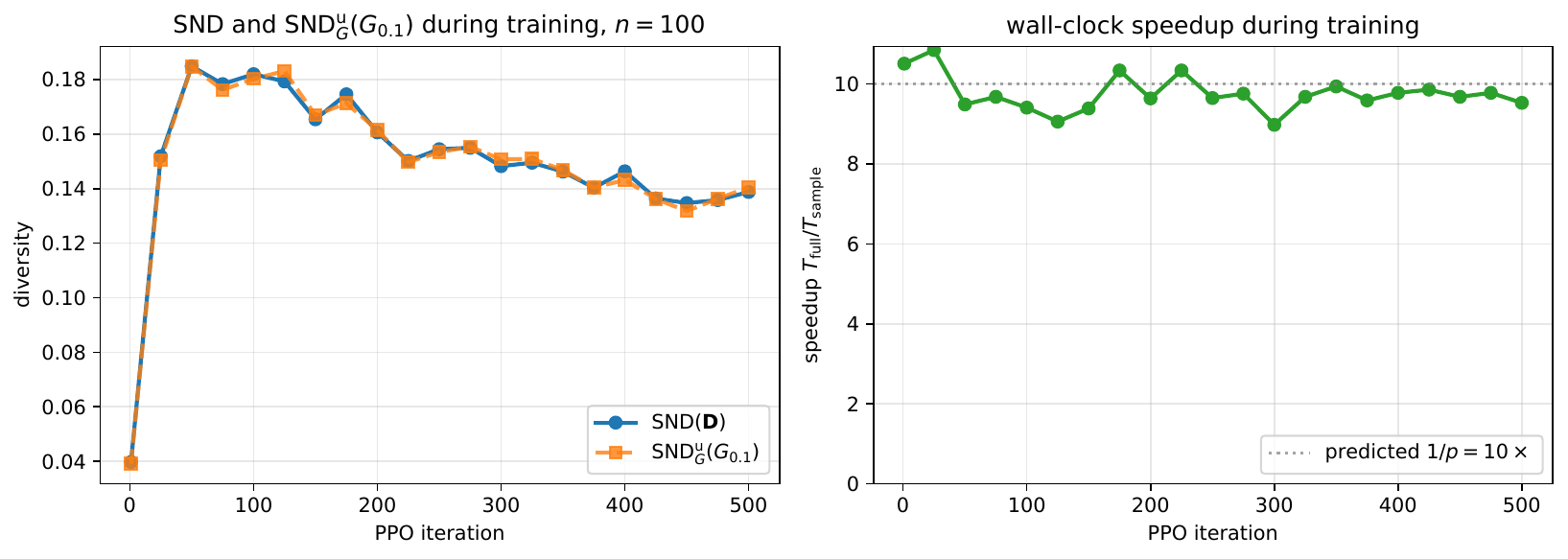}
\caption{$n=100$ PPO training run. Bernoulli-$0.1$ Graph-SND tracks full
SND at every logged iteration while reducing metric time by about
$10\times$, matching Proposition~\ref{prop:complexity} under
non-stationary training dynamics.}
\label{fig:scaling-n100}
\end{figure}

\begin{figure}[t]
\centering
\includegraphics[width=\linewidth]{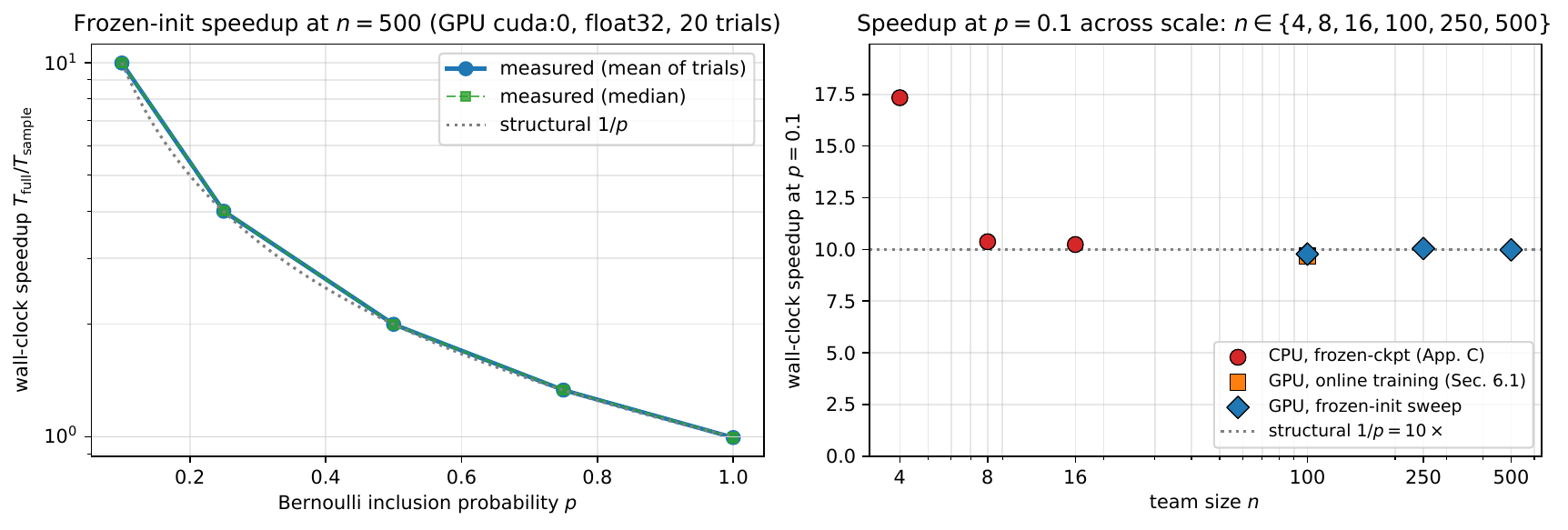}
\caption{Wall-clock scaling of $\SNDG^{\mathrm{u}}(G_p)$ versus full
SND across $n\in\{4,\ldots,500\}$. Bernoulli-$p$ sampling follows the
structural $1/p$ speedup, and Bernoulli-$0.1$ remains near
$10\times$ through $n=500$; sampled timings include edge sampling and
edge-transfer overhead.}
\label{fig:timing-n500}
\end{figure}

\subsection{Closed-loop diversity control on VMAS Dispersion}
\label{sec:experiments:dico-dispersion}

We integrate Graph-SND into the DiCo diversity controller
\citep{bettini2024dico} on VMAS Dispersion. An $n=10$ study shows that
$k$-NN, Bernoulli-$0.1$, and random $d$-regular Graph-SND all track
$\SND_{\mathrm{des}}=0.1$ and match the full-SND controller's reward
within seed variation while reducing per-call metric cost
(Figure~\ref{fig:dico-dispersion-knn}). The stronger closed-loop
evidence is the $n=50$ head-to-head in
Table~\ref{tab:dico-n50-main}: each of three seeds and three set points
is trained twice, once with Bernoulli-$0.1$ Graph-SND and once with full
SND. Bernoulli-$0.1$ tracks all targets to sub-$1\%$ relative error,
yields paired reward differences centered near zero, and reduces metric
cost by $9.24$ to $9.63\times$, matching the $10\times$ pair-count
prediction. This is the main closed-loop substitution test: full SND is
not merely evaluated post hoc, but trained as the competing DiCo signal
on the same seed and set-point grid. A separate post-hoc full-SND audit
of the trained actions (Table~\ref{tab:dico-n50-posthoc-main} and
Appendix~\ref{app:dico-n50-sweep},
Figure~\ref{fig:dico-n50-posthoc-full-snd}) confirms that the
Bernoulli-controlled policies' complete-graph SND also stays within
$0.61\%$ of the target; this audit signal is never fed back into the
controller. The only systematic difference is the aggregation used by
the controller, and the observed differences are smaller than
seed-to-seed variation.

\begin{figure}[t]
\centering
\includegraphics[width=\linewidth]{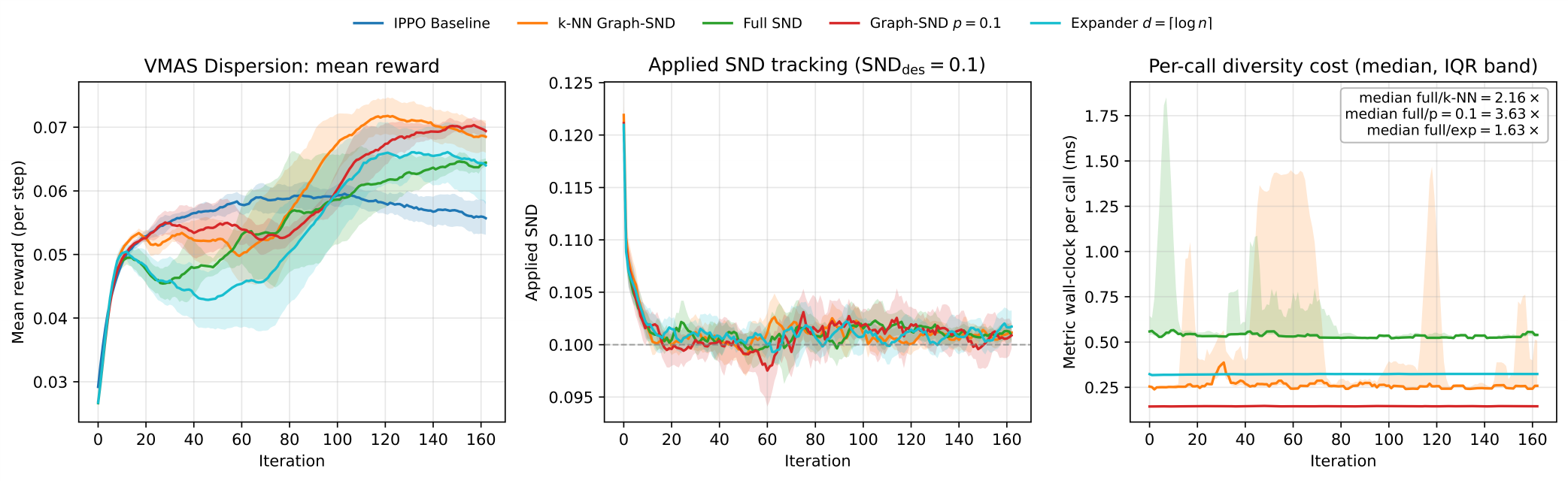}
\caption{DiCo with Graph-SND on VMAS Dispersion ($n{=}10$, ten foods,
$\SND_{\mathrm{des}}{=}0.1$, three seeds). Sparse $k$-NN,
Bernoulli-$0.1$, and expander aggregators track the full-SND controller
while reducing per-call metric cost.}
\label{fig:dico-dispersion-knn}
\end{figure}

\begin{table}[t]
\centering
\small
\begin{tabular}{lccc}
\toprule
$\SND_{\mathrm{des}}$ & Rel.\ tracking err.\ (Bern/full) & Reward (Bern/full) & Metric speedup \\
\midrule
$0.12$ & $0.53\%$ / $0.37\%$ & $0.295{\pm}0.016$ / $0.325{\pm}0.008$ & $9.63\times$ \\
$0.14$ & $0.48\%$ / $0.31\%$ & $0.338{\pm}0.008$ / $0.327{\pm}0.007$ & $9.39\times$ \\
$0.15$ & $0.42\%$ / $0.35\%$ & $0.350{\pm}0.005$ / $0.333{\pm}0.009$ & $9.24\times$ \\
\bottomrule
\end{tabular}
\caption{Compact main-text summary of the $n{=}50$ DiCo head-to-head
against full SND. Each row aggregates three matched seeds on the late
window (last $50$ iterations); Bernoulli-$0.1$ uses
$\E[|E|]=122.5$ sampled edges versus $1{,}225$ full pairs. Full curves
and per-cell statistics appear in Appendix~\ref{app:dico-n50-sweep};
across all nine matched cells, the paired mean reward difference
(Bernoulli minus full) is $-6.2{\times}10^{-4}\pm9.2{\times}10^{-3}$
SEM.
The largest single-cell mean gap is ${\sim}9\%$ at $\SND_{\mathrm{des}}{=}0.12$, offset by Bernoulli-$0.1$ leading at the two higher set points; the paired analysis averages over this variation.}
\label{tab:dico-n50-main}
\end{table}

\begin{table}[t]
\centering
\small
\begin{tabular}{lcc}
\toprule
$\SND_{\mathrm{des}}$ & Post-hoc full-SND err.\ (Bern/full) &
Post-hoc minus applied (Bern/full) \\
\midrule
$0.12$ & $0.50\%$ / $0.48\%$ &
$2.9{\times}10^{-4}$ / $3.0{\times}10^{-4}$ \\
$0.14$ & $0.61\%$ / $0.36\%$ &
$4.7{\times}10^{-4}$ / $2.6{\times}10^{-4}$ \\
$0.15$ & $0.43\%$ / $0.46\%$ &
$3.1{\times}10^{-4}$ / $4.7{\times}10^{-4}$ \\
\bottomrule
\end{tabular}
\caption{Post-hoc complete-graph SND audit for the $n{=}50$ DiCo
head-to-head. Each cell is remeasured with full SND on the scaled
actions at every iteration, but the audit value is not used by the
controller. Bernoulli-$0.1$ remains within $0.61\%$ of the desired
complete-graph diversity at all set points, and the audit differs from
the online applied signal by less than $5{\times}10^{-4}$.}
\label{tab:dico-n50-posthoc-main}
\end{table}

\subsection{Expander sparsification ablation}
\label{sec:experiments:expander}

We test the deterministic fixed-$G$ theory by comparing random
$d$-regular graphs to Bernoulli, uniform-size, and $k$-NN graphs at
matched edge counts
$d\in\{\lceil\log_2 n\rceil,\lceil2\log_2 n\rceil,\lceil4\log_2 n\rceil\}$,
$n\in\{50,100,200,500\}$, five graph seeds each. Figure~\ref{fig:expander-distortion}
shows the load-bearing result: $d$-regular expanders keep
$\SNDG^{\mathrm{u}}/\SND$ in $[0.9987,1.0013]$ at every scale, including
$n=500$ with $d=9$ and only $2{,}250$ of $124{,}750$ pairs. Their edge
forwarding index follows the predicted $\mathcal{O}(n\log n/d)$ scaling
and is $2$ to $39\times$ lower than matched-edge baselines at $n=500$,
explaining why the forwarding-index bound favors expanders even though
the worst-case $\mathcal{O}(\log n)$ ratio is conservative. The trained
$n=100$ checkpoint panel preserves the same ordering
(Appendix~\ref{app:expander-extra}). A non-Gaussian PettingZoo MPE
simple-spread transfer panel with categorical policies and TVD distance
also shows small bias and matched cost reduction
(Appendix~\ref{app:verify}, Table~\ref{tab:mpe-tvd-panel}).

The expander claim is deliberately split into worst-case and empirical
parts. The forwarding-index theorem gives a deterministic guarantee for
arbitrary metric distance tables, not a promise of near-unit distortion.
The near-unit ratios in Figure~\ref{fig:expander-distortion} are a
data-dependent observation, partially explained by
Proposition~\ref{prop:spectral-nuclear-discrepancy}: when the realized
distance matrix has controlled normalized nuclear norm, expansion turns
the logarithmic worst-case dependence into a decreasing spectral
discrepancy bound. This distinction matters because fixed sparse graphs
change the measurement target unless their structure is justified; the
expander result shows one fixed-graph choice that remains accurate on
the tested MARL distance matrices while retaining near-linear edge count.

\begin{figure}[t]
  \centering
  \includegraphics[width=\textwidth]{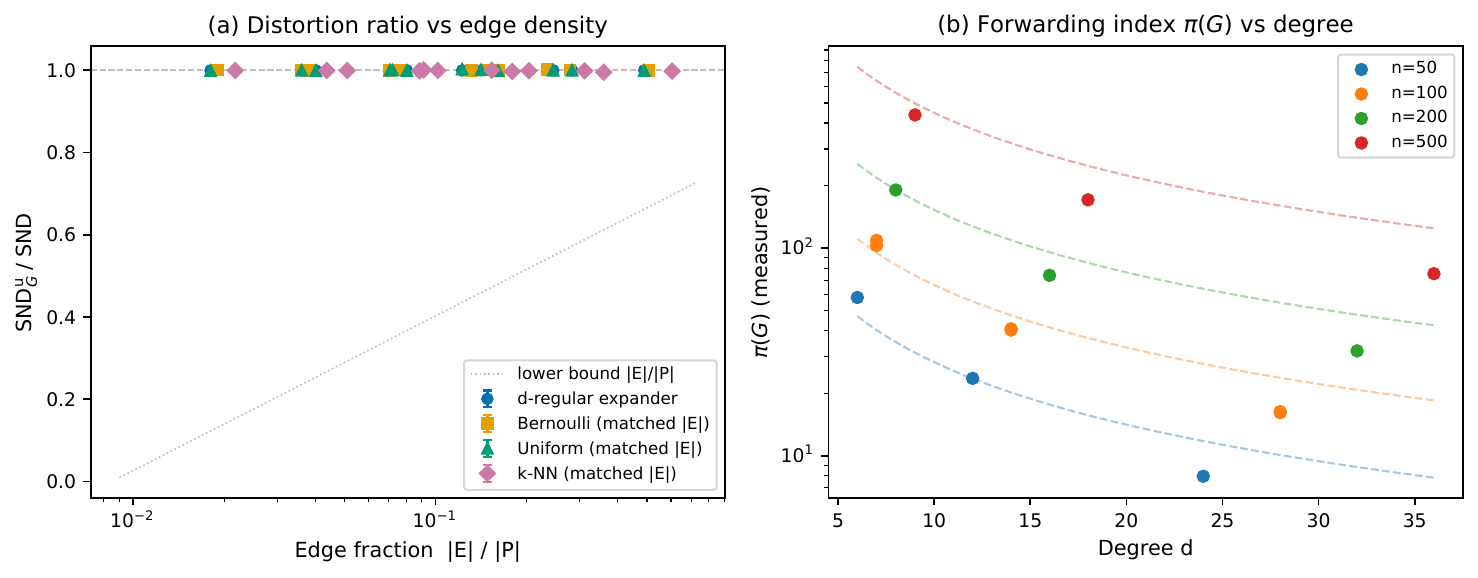}
  \caption{Expander sparsification ablation. Random $d$-regular graphs
  give the tightest near-unit $\SNDG^{\mathrm{u}}/\SND$ ratios at
  matched edge budgets, and their measured forwarding index tracks the
  $\mathcal{O}(n\log n/d)$ mechanism in Theorem~\ref{thm:distortion}.}
  \label{fig:expander-distortion}
\end{figure}

% ============================================================
\section{Discussion}
\label{sec:discussion}

Graph-SND is best understood as a choice of graph semantics. Use
Bernoulli-$p$ when full $\SND$ is the target and sparsity is purely
computational; use $k$-NN or communication graphs when local
interactions define the intended diversity signal; use random
$d$-regular expanders when a fixed near-linear aggregator with strong
worst-case structure is desired. The same graph layer applies beyond
Gaussian Wasserstein SND whenever the pairwise dissimilarity satisfies
the relevant boundedness or metric assumptions (Appendix~\ref{app:verify},
Table~\ref{tab:mpe-tvd-panel}).

\begin{center}
\refstepcounter{table}
\label{tab:graph-selection-rule}
\centering
\small
\begin{tabular}{p{0.20\linewidth}p{0.25\linewidth}p{0.42\linewidth}}
\toprule
Goal & Recommended $G$ & Main guarantee and use case \\
\midrule
Estimate full $\SND$ & Bernoulli-$p$ or fixed-size random sample &
Horvitz-Thompson unbiasedness and $O(1/\sqrt m)$ concentration; best
when full-system diversity remains the target. \\
Measure local diversity & Communication, spatial, or $k$-NN graph &
Interpretable fixed-$G$ signal over interacting pairs; appropriate when
non-edges are intentionally excluded from the semantic target. \\
Use a fixed sparse proxy & Random $d$-regular expander,
$d=\Theta(\log n)$ & Forwarding-index worst-case structure plus
near-unit empirical ratios at $\Theta(n\log n)$ edges in our ablation. \\
\bottomrule
\end{tabular}
\par\smallskip
\begin{minipage}{0.96\linewidth}
\small \textbf{Table~\thetable:} Practical graph-selection rule. The graph is
not a nuisance parameter: it decides whether Graph-SND is estimating full
SND or defining a different, localized diversity functional.
\end{minipage}
\end{center}

\paragraph{Why aggregation is the right intervention.}
Graph-SND is most useful when SND is not a one-off diagnostic but a
quantity repeatedly evaluated inside training, selection, or diversity
control. In that setting full SND imposes a quadratic tax at every metric
call, while sparse aggregation exposes an explicit accuracy-cost knob:
$p$ or $m$ for random estimators, $k$ for local neighborhoods, and $d$
for expander proxies. Because the intervention is only the aggregation
layer, the downstream optimizer and pairwise distance remain unchanged.
The evidence is therefore deliberately staged: exact recovery and
unbiasedness establish equivalence on the complete target, scaling
experiments show the expected cost law, and DiCo verifies that the same
replacement remains stable when the metric is used as feedback rather
than only measured offline.

\paragraph{Failure modes and interpretation.}
The main failure mode is transparent from the graph semantics. A sparse
fixed graph can hide diversity on non-edges: for example, if agents form
two behaviorally distinct clusters but $G$ contains mostly within-cluster
edges, $\SNDG$ may be small while complete-graph $\SND$ is large. This is
not a failure of the estimator theory; it means the graph has selected a
local measurement target. Conversely, Bernoulli and fixed-size random
graphs treat omitted edges as missing samples from the complete pair set,
so their guarantees speak directly about full SND. This separation is why
we evaluate fixed expanders as deterministic proxies, evaluate random
graphs by unbiasedness and concentration, and audit the DiCo controller
post hoc with full complete-graph SND. The same implementation can
therefore support two scientifically different uses: measuring diversity
only where interaction structure says it matters, or approximating the
global SND metric when all-pairs aggregation is the bottleneck.

\paragraph{Limitations and future work.}
The empirical scope remains intentionally focused. Online training is
shown at $n=100$ for PPO and at $n=50$ for DiCo, while $n=500$ is a
frozen-policy timing and expander-measurement scale. All RL runs use
Independent PPO, with centralized-critic algorithms left for future
work. Non-Gaussian validation is limited to one PettingZoo MPE
TVD/categorical panel and synthetic checks, even though the main
estimator guarantees are distance-agnostic under boundedness. Finally,
Proposition~\ref{prop:spectral-nuclear-discrepancy} explains part of the
gap between worst-case expander distortion and observed near-unit ratios
via nuclear-norm control, and
Proposition~\ref{prop:probabilistic-distortion} gives an unconditional
$\widetilde{\mathcal{O}}(1/\sqrt{n})$ probabilistic bound for the random
$d$-regular sampling distribution; the residual gap to the empirically
observed $\mathcal{O}(1/n)$-scale ratios is consistent with standard
sub-Gaussian-vs-Bernstein looseness and is not closed by these tools.
Natural next steps are learned or task-adaptive graphs, larger
discrete-action benchmarks, and broader DiCo studies over tasks,
set points, and sampling schedules.
\section{Conclusion}
\label{sec:conclusion}

Graph-SND turns SND's all-pairs average into a graph-edge aggregate. This
single substitution unifies local diversity measurement, unbiased sparse
estimation, and expander-based deterministic approximation, with theory
and experiments showing the expected cost reduction while preserving the
behavior of SND in training and DiCo control. The practical knob is the
graph $G$: it decides whether the diversity signal is local by design or
a cheaper approximation of the global complete-graph metric.
The implication is that all-pairs aggregation should be a deliberate
choice, not the default paid at every scale: when global SND is required,
random sampling gives estimator guarantees; when task structure matters,
fixed graphs give a sharper local signal; and when deterministic sparse
proxies are desired, expanders provide a strong near-linear candidate.
\codelocation

\clearpage
\bibliographystyle{plainnat}
\bibliography{references}

\begin{ack}
The author thanks Matteo Bettini and the ProrokLab for the System
Neural Diversity implementation and the VMAS \& DiCo open-source
releases on which this work builds.
\end{ack}

% ============================================================
% Appendix: proofs, hyperparameters, and empirical verification.
\appendix
\section{Proofs}
\label{app:proofs}

We use the conventions of Section~\ref{sec:background}: $\Agents$ is the
agent set with $|\Agents| = n$; $\Pairs = \binom{\Agents}{2}$ is the set
of unordered agent pairs, with $|\Pairs| = \binom{n}{2}$; $d(i,j) \geq 0$
is the Monte Carlo behavioral distance~\eqref{eq:dij}, treated as a
deterministic scalar given the rollout set $\Rollouts$; and $G = (\Agents,
E, w)$ has non-negative edge weights $w_{ij}$ with total weight
$W(G) = \sum_{\{i,j\} \in E} w_{ij} > 0$.

\subsection*{Proof of Proposition~\ref{prop:recovery} (Recovery)}

If $G = K_n$ with $w_{ij} = 1$ for all $\{i,j\} \in \Pairs$, then
$E = \Pairs$ and $W(G) = |\Pairs| = \binom{n}{2}$. Substituting
into~\eqref{eq:graph-snd},
\[
\SNDG(\mathbf{D}, K_n)
\;=\; \frac{1}{\binom{n}{2}} \sum_{\{i,j\} \in \Pairs} 1 \cdot d(i,j)
\;=\; \binom{n}{2}^{-1} \sum_{i < j} d(i,j)
\;=\; \SND(\mathbf{D}). \qed
\]

\subsection*{Proof of Proposition~\ref{prop:nonneg} (Non-negativity)}

Each term $w_{ij} d(i,j)$ in~\eqref{eq:graph-snd} is a product of
non-negative factors, so the numerator $\sum_{E} w_{ij} d(i,j) \geq 0$.
Since $W(G) > 0$ by assumption, $\SNDG(\mathbf{D}, G) \geq 0$.

For the equality condition, $\SNDG = 0$ iff
$\sum_{\{i,j\} \in E} w_{ij} d(i,j) = 0$. A sum of non-negative terms is
zero iff each term is zero, so the equality holds iff $w_{ij} d(i,j) = 0$
for every $\{i,j\} \in E$. Equivalently, $d(i,j) = 0$ for every edge
$\{i,j\} \in E$ with $w_{ij} > 0$. \qed

\subsection*{Proof of Proposition~\ref{prop:perm} (Automorphism invariance)}

Let $\sigma : \Agents \to \Agents$ be an automorphism of $(G, w)$, meaning
$\sigma$ is a bijection satisfying $\{i,j\} \in E \iff \{\sigma(i),
\sigma(j)\} \in E$ and $w_{ij} = w_{\sigma(i)\sigma(j)}$ for all
$\{i,j\} \in E$. Then
\begin{align*}
\SNDG(\sigma \cdot \mathbf{D}, G)
&\;=\; \frac{1}{W(G)} \sum_{\{i,j\} \in E} w_{ij}\, d(\sigma(i), \sigma(j)).
\end{align*}
Reindex the sum by $\{i', j'\} = \{\sigma(i), \sigma(j)\}$. Because
$\sigma$ is an automorphism, the map $\{i,j\} \mapsto \{\sigma(i),
\sigma(j)\}$ is a bijection $E \to E$, and $w_{ij} =
w_{\sigma(i)\sigma(j)} = w_{i'j'}$. Hence
\[
\SNDG(\sigma \cdot \mathbf{D}, G)
\;=\; \frac{1}{W(G)} \sum_{\{i',j'\} \in E} w_{i'j'}\, d(i', j')
\;=\; \SNDG(\mathbf{D}, G). \qed
\]

\subsection*{Proof of Proposition~\ref{prop:complexity} (Complexity)}

Evaluating~\eqref{eq:graph-snd} on a precomputed graph $G$ reduces to
computing $d(i,j)$ for each $\{i,j\} \in E$ and forming a weighted sum. The
sum costs $O(|E|)$ floating-point operations, dominated by the $|E|$
pairwise-distance computations. $\SND$ in~\eqref{eq:snd} instead requires
$|\Pairs| = \binom{n}{2}$ such computations. Since each pairwise-distance
computation has the same cost in both cases, the speedup factor in the
number of pairwise-distance evaluations is $\binom{n}{2}/|E|$. For $k$-NN
graphs $|E| = O(nk)$, giving a factor of $\Theta(n/k)$. This statement
does not include the cost of constructing $G$. \qed

\subsection*{Proof of Proposition~\ref{prop:unbiased} (Unbiasedness)}

For each $\{i,j\} \in \Pairs$, let $Z_{ij} = \ind[\{i,j\} \in E(G_p)]$.
By construction, $\{Z_{ij}\}_{\{i,j\} \in \Pairs}$ are independent
Bernoulli$(p)$ random variables. With $w_{ij} = 1/p$ on edges of $G_p$,
\[
\widehat{\SND}_{\mathrm{HT}}(G_p)
\;=\; |\Pairs|^{-1} \sum_{\{i,j\} \in E(G_p)} w_{ij}\, d(i,j)
\;=\; |\Pairs|^{-1} \sum_{\{i,j\} \in \Pairs} Z_{ij} \cdot \tfrac{1}{p}
\cdot d(i,j).
\]
Taking expectation and using $\E[Z_{ij}] = p$,
\[
\E_{G_p}\!\left[\widehat{\SND}_{\mathrm{HT}}(G_p)\right]
\;=\; |\Pairs|^{-1} \sum_{\{i,j\} \in \Pairs}
p \cdot \tfrac{1}{p} \cdot d(i,j)
\;=\; |\Pairs|^{-1} \sum_{\{i,j\} \in \Pairs} d(i,j)
\;=\; \SND(\mathbf{D}). \qed
\]

\subsection*{Proof of Theorem~\ref{thm:concentration} (Concentration)}

We first establish that, conditional on $|E(G_p)| = m$, $E(G_p)$ is a
uniform random subset of $\Pairs$ of size $m$.

\begin{lemma}
\label{lem:conditional-uniform}
Let $\{Z_\alpha\}_{\alpha \in \Pairs}$ be i.i.d.\ Bernoulli$(p)$ with
$p \in (0,1)$, and let $E = \{\alpha : Z_\alpha = 1\}$. Conditional on
$|E| = m$ with $0 \leq m \leq |\Pairs|$, the set $E$ is uniformly
distributed over subsets of $\Pairs$ of size $m$.
\end{lemma}
\begin{proof}
For any fixed $T \subseteq \Pairs$ with $|T| = m$,
\[
\Prob(E = T \mid |E| = m)
\;=\; \frac{\Prob(E = T)}{\Prob(|E| = m)}
\;=\; \frac{p^m (1-p)^{|\Pairs|-m}}
          {\binom{|\Pairs|}{m} p^m (1-p)^{|\Pairs|-m}}
\;=\; \binom{|\Pairs|}{m}^{-1},
\]
which does not depend on $T$.
\end{proof}

Now fix $m \geq 1$ and condition on $|E(G_p)| = m$. By
Lemma~\ref{lem:conditional-uniform}, $E(G_p)$ is a uniform random sample
of size $m$ without replacement from the finite population
\[
\Pi \;:=\; \bigl\{ d(i,j) : \{i,j\} \in \Pairs \bigr\},
\]
where $|\Pi| = |\Pairs|$ and each $d(i,j) \in [0, D_{\max}]$. The
population mean is
\[
\mu \;:=\; |\Pairs|^{-1} \sum_{\{i,j\} \in \Pairs} d(i,j) \;=\; \SND(\mathbf{D}),
\]
and the sample mean is precisely
\[
\SNDG^{\mathrm{u}}(\mathbf{D}, G_p) \;=\; m^{-1} \sum_{\{i,j\} \in E(G_p)} d(i,j).
\]

\citet[Section 6, Theorem 4]{hoeffding1963probability} establishes that
the two-sided Hoeffding inequality applies to sampling without
replacement from a bounded finite population: for any $t > 0$,
\begin{equation}
\label{eq:hoeff-wor}
\Prob\bigl(\abs{\SNDG^{\mathrm{u}} - \SND} \geq t \;\bigm|\; |E(G_p)| = m\bigr)
\;\leq\; 2\exp\!\left(- \frac{2 m t^2}{D_{\max}^2}\right).
\end{equation}
Setting the right-hand side equal to $\delta$ and solving for $t$ gives
\[
t \;=\; D_{\max}\sqrt{\frac{\log(2/\delta)}{2m}}.
\]
Thus with probability at least $1 - \delta$ conditional on $|E(G_p)| = m$,
\[
\abs{\SNDG^{\mathrm{u}}(\mathbf{D}, G_p) - \SND(\mathbf{D})}
\;\leq\; D_{\max}\sqrt{\frac{\log(2/\delta)}{2m}}. \qed
\]

\subsection*{Deterministic fixed-$G$ distortion}
\label{app:distortion}

For a connected unit-weight graph $G=(\Agents,E)$, a \emph{path system}
is an assignment of a simple $i$-$j$ path $\mathcal{R}_{ij}\subseteq E$
to every pair $\{i,j\}\in\Pairs$, with
$\mathcal{R}_{ij}=\{\{i,j\}\}$ whenever $\{i,j\}\in E$. The
\emph{congestion} of a path system $\mathcal{R}$ at an edge
$\{a,b\}\in E$ is
$c_{\mathcal{R}}(a,b):=\abs{\{\{i,j\}\in\Pairs :
\{a,b\}\in\mathcal{R}_{ij}\}}$, and the \emph{edge forwarding index} of
$G$ is
$\pi(G):=\min_{\mathcal{R}} \max_{\{a,b\}\in E} c_{\mathcal{R}}(a,b)$
\citep{chung1987forwarding,heydemann1989forwarding}. The index depends
on $G$ alone, satisfies $\pi(K_n)=1$, admits closed forms for many
graph families (hypercubes, tori, Cayley graphs), and is
$\mathcal{O}(n\log n / d)$ for $d$-regular graphs with bounded spectral
gap~\citep{leighton1999multicommodity}.

\begin{proposition}[Spectral discrepancy under nuclear-norm control]
\label{prop:spectral-nuclear-discrepancy}
Let $G$ be a connected simple $d$-regular graph on $n$ vertices with adjacency
matrix $A$ and
$\lambda_2:=\max\{|\lambda_i(A)|: i\geq 2\}$, where
$d=\lambda_1(A)$ is the trivial eigenvalue. Let
$\mathbf{D}\in\R^{n\times n}$ be symmetric with zero diagonal. Then
\begin{equation}
\label{eq:spectral-nuclear-absolute}
\abs{\SNDG^{\mathrm{u}}(\mathbf{D},G)-\SND(\mathbf{D})}
\leq
\frac{\lambda_2+d/(n-1)}{n\,d}\,\|\mathbf{D}\|_* ,
\end{equation}
where $\|\cdot\|_*$ denotes the nuclear norm. If $\SND(\mathbf{D})>0$
and
\[
\rho_*(\mathbf{D})
:=
\frac{\|\mathbf{D}\|_*}{n\,\SND(\mathbf{D})},
\]
then
\begin{equation}
\label{eq:spectral-nuclear-relative}
\abs{\frac{\SNDG^{\mathrm{u}}(\mathbf{D},G)}{\SND(\mathbf{D})}-1}
\leq
\frac{\lambda_2+d/(n-1)}{d}\,\rho_*(\mathbf{D}).
\end{equation}
In particular, if $G$ is Ramanujan
($\lambda_2\leq 2\sqrt{d-1}$) and $\rho_*(\mathbf{D})\leq C$, then the
relative error is at most
$C(2\sqrt{d-1}+d/(n-1))/d=\mathcal{O}(C/\sqrt d)$; at
$d=\Theta(\log n)$ this is
$\mathcal{O}(C/\sqrt{\log n})$.
\end{proposition}

\begin{remark}[What the spectral bound explains]
\label{rem:spectral-sharpening}
Section~\ref{sec:experiments:expander} measures
$\SNDG^{\mathrm{u}}/\SND\in[0.9987,1.0013]$ on random $d$-regular
graphs at $d{=}\Theta(\log n)$, far tighter than the worst-case
$\mathcal{O}(\log n)$ ratio of
Corollary~\ref{cor:expander-distortion}.
Proposition~\ref{prop:spectral-nuclear-discrepancy} formalizes a
distribution-dependent explanation: if the realized pairwise
dissimilarity matrix has controlled normalized nuclear norm
$\rho_*(\mathbf{D})$, then spectral expansion improves the deterministic
bound from logarithmic distortion to a decreasing
$\mathcal{O}(\rho_*(\mathbf{D})/\sqrt{\log n})$ relative-error bound at
$d=\Theta(\log n)$. This is still a fixed-$G$, worst-case alignment
bound over matrices satisfying the nuclear-norm condition; it does not
fully explain the much smaller empirical ratios, which likely also use
random-graph variance reduction and the specific distributional
structure of MARL distance matrices.
Proposition~\ref{prop:probabilistic-distortion} below gives an
unconditional probabilistic sharpening for random $d$-regular graphs;
closing the remaining gap to the empirically observed
$\mathcal{O}(1/n)$-scale ratios would require sharper
distribution-dependent or Bernstein-type analysis.
\end{remark}

\paragraph{Proof of Proposition~\ref{prop:spectral-nuclear-discrepancy}.}
Let $J=\mathbf{1}\mathbf{1}^{\top}$ and define the discrepancy matrix
\[
M_G:=A-\frac{d}{n-1}(J-I).
\]
For the all-ones vector, $A\mathbf{1}=d\mathbf{1}$ and
$(J-I)\mathbf{1}=(n-1)\mathbf{1}$, so $M_G\mathbf{1}=0$.
For any eigenvector $v\perp\mathbf{1}$ of $A$ with
$Av=\lambda v$, we have $Jv=0$ and $(J-I)v=-v$, hence
\[
M_G v=\left(\lambda+\frac{d}{n-1}\right)v.
\]
Therefore
\[
\|M_G\|_{\mathrm{op}}
\leq
\lambda_2+\frac{d}{n-1}.
\]
Using $D_{ii}=0$, $|E|=nd/2$, and
$|\Pairs|=n(n-1)/2$,
\[
\langle M_G,\mathbf{D}\rangle_F
=
2|E|\,\SNDG^{\mathrm{u}}(\mathbf{D},G)
-\frac{d}{n-1}\,2|\Pairs|\,\SND(\mathbf{D})
=
nd\bigl(\SNDG^{\mathrm{u}}(\mathbf{D},G)-\SND(\mathbf{D})\bigr).
\]
Operator-norm/nuclear-norm duality gives
$|\langle M_G,\mathbf{D}\rangle_F|
\leq \|M_G\|_{\mathrm{op}}\|\mathbf{D}\|_*$, proving
\eqref{eq:spectral-nuclear-absolute}. Dividing by $\SND(\mathbf{D})$
proves \eqref{eq:spectral-nuclear-relative}; the Ramanujan specialization
substitutes $\lambda_2\leq2\sqrt{d-1}$.

\paragraph{Proof of Theorem~\ref{thm:distortion}.}

\emph{Lower bound.} Since $d(i,j) \geq 0$, dropping the
$|\Pairs|-|E|$ non-edge terms from the full sum cannot increase it:
\[
|\Pairs|\cdot\SND(\mathbf{D})
\;=\; \sum_{\{i,j\}\in\Pairs} d(i,j)
\;\geq\; \sum_{\{i,j\}\in E} d(i,j)
\;=\; |E|\cdot\SNDG^{\mathrm{u}}(\mathbf{D},G),
\]
with equality iff $d(i,j)=0$ for every non-edge
$\{i,j\}\in\Pairs\setminus E$.

\emph{Upper bound.} Fix a path system $\mathcal{R}$ on $G$ in the sense
of Section~\ref{sec:theory}. For each edge pair $\{i,j\}\in E$, the
single-edge path $\mathcal{R}_{ij}=\{\{i,j\}\}$ gives the identity
$d(i,j)=\sum_{\{a,b\}\in\mathcal{R}_{ij}} d(a,b)$. For each non-edge
pair $\{i,j\}\in\Pairs\setminus E$, writing the path as
$\mathcal{R}_{ij}=(i=v_0,v_1,\dots,v_L=j)$ and applying the triangle
inequality $L-1$ times,
\[
d(i,j)\;\leq\; d(v_0,v_1)+d(v_1,v_L)\;\leq\;\cdots\;\leq\;
\sum_{k=0}^{L-1} d(v_k,v_{k+1})
\;=\; \sum_{\{a,b\}\in\mathcal{R}_{ij}} d(a,b).
\]
Combining both cases and exchanging the order of summation,
\[
\sum_{\{i,j\}\in\Pairs} d(i,j)
\;\leq\; \sum_{\{i,j\}\in\Pairs}\sum_{\{a,b\}\in\mathcal{R}_{ij}} d(a,b)
\;=\; \sum_{\{a,b\}\in E} c_{\mathcal{R}}(a,b)\,d(a,b),
\]
where
$c_{\mathcal{R}}(a,b):=|\{\{i,j\}\in\Pairs : \{a,b\}\in\mathcal{R}_{ij}\}|$
is the congestion of $\mathcal{R}$ at edge $\{a,b\}$. Since $d(a,b)\geq 0$,
\[
\sum_{\{a,b\}\in E} c_{\mathcal{R}}(a,b)\,d(a,b)
\;\leq\; \Bigl(\max_{\{a,b\}\in E} c_{\mathcal{R}}(a,b)\Bigr)
\sum_{\{a,b\}\in E} d(a,b).
\]
This inequality holds for every path system $\mathcal{R}$, so it holds
in particular for one that attains the minimum in the definition of
$\pi(G)$, for which the bracket equals $\pi(G)$. Therefore
\[
|\Pairs|\cdot\SND(\mathbf{D})
\;\leq\; \pi(G)\cdot\sum_{\{a,b\}\in E} d(a,b)
\;=\; \pi(G)\cdot |E|\cdot\SNDG^{\mathrm{u}}(\mathbf{D}, G).
\]
Dividing by $|\Pairs|$ yields
$\SND(\mathbf{D})\leq (|E|\pi(G)/|\Pairs|)\,\SNDG^{\mathrm{u}}(\mathbf{D},G)$.
The recovery statement follows because $K_n$ satisfies $\pi(K_n)=1$
(take $\mathcal{R}_{ij}=\{\{i,j\}\}$ for all pairs) and
$|E|=|\Pairs|=\binom{n}{2}$. \qed

\paragraph{Fractional sharpening.}
Replacing the integral path system by a convex combination of paths
$f_{ij}=\sum_{\mathcal{R}\in\Pi_{ij}}\alpha_{\mathcal{R}}\mathbf{1}_{\mathcal{R}}$
($\alpha_{\mathcal{R}}\geq 0$, $\sum_{\mathcal{R}}\alpha_{\mathcal{R}}=1$)
and repeating the argument with the fractional congestion
$C^{*}_{\mathbf{f}}(a,b):=\sum_{\{i,j\}\in\Pairs} f_{ij}(a,b)$ tightens
the upper bound of Eq.~\eqref{eq:distortion} to
$|E|\pi^{*}(G)/|\Pairs|$, where
$\pi^{*}(G)=\min_{\mathbf{f}}\max_{\{a,b\}\in E} C^{*}_{\mathbf{f}}(a,b)\leq\pi(G)$
is the fractional edge forwarding index, computable by a
multicommodity-flow LP~\citep{leighton1999multicommodity}.

\subsection*{Probabilistic distortion for random $d$-regular graphs}

\begin{proposition}[Unconditional probabilistic distortion]
\label{prop:probabilistic-distortion}
Let $G$ be drawn uniformly from simple $d$-regular graphs on $n$
vertices, with $d\geq 3$, $nd$ even, and $d^2 = o(n)$ (so the headline
setting $d=\Theta(\log n)$ is included). Let $\mathbf{D}\in\R^{n\times n}$
be symmetric with $D_{ii}=0$ and $D_{ij}\in[0,D_{\max}]$. Then for any
$\delta\in(0,1)$, with probability at least $1-\delta$ over $G$,
\begin{equation}
\label{eq:probabilistic-distortion}
\bigl|\,\SNDG^{\mathrm{u}}(\mathbf{D},G)-\SND(\mathbf{D})\,\bigr|
\;\leq\;
D_{\max}\sqrt{\frac{8}{nd}\left(\log\frac{4}{\delta}
+\frac{d^2-1}{4}\right)}.
\end{equation}
In particular, for $d=\Theta(\log n)$, this gives
$|\SNDG^{\mathrm{u}}-\SND|
=\widetilde{\mathcal{O}}(D_{\max}/\sqrt{n})$ with high probability,
independent of any structural assumption on $\mathbf{D}$ beyond
boundedness.
\end{proposition}

\begin{proof}
The argument uses three classical ingredients.

\emph{Step 1: Configuration model.}
Represent the uniform random $d$-regular multigraph via the configuration
model~\citep{bollobas1980probabilistic}: assign $d$ labeled stubs to each
vertex and draw a uniformly random perfect matching $\pi$ of the $nd$
stubs. Define $f(\pi):=\sum_{e\in E(G(\pi))}D_e$, so
$\SNDG^{\mathrm{u}}(\mathbf{D},G(\pi))=2f(\pi)/(nd)$. By edge symmetry,
$\E_\pi[f(\pi)]=(nd/2)\,\SND(\mathbf{D})$.

\emph{Step 2: McDiarmid's permutation inequality.}
For matchings $\pi,\pi'$ differing by a single transposition of two stubs
in different pairs, the edge sets differ in exactly four edges (two
deleted, two added), so $|f(\pi)-f(\pi')|\leq 2D_{\max}$. By Hoeffding's
inequality for permutations~\citep[\S6.1]{mcdiarmid1989method} with
$N=nd$ and bounded difference $c=2D_{\max}$:
\[
\Pr_\pi\!\bigl[|f(\pi)-\E f(\pi)|\geq t\bigr]
\;\leq\; 2\exp\!\left(-\frac{t^2}{2\,nd\,D_{\max}^2}\right).
\]
Setting $t=(nd/2)\epsilon$ and dividing through:
\[
\Pr_\pi\!\bigl[|\SNDG^{\mathrm{u}}-\SND|\geq\epsilon\bigr]
\;\leq\; 2\exp\!\left(-\frac{nd\,\epsilon^2}{8\,D_{\max}^2}\right).
\]

\emph{Step 3: Conditioning on simplicity.}
The uniform simple $d$-regular graph is the configuration multigraph
conditioned on being simple. By~\citet{bollobas1980probabilistic}
and~\citet{mckay1981subgraphs}, in the sparse regime $d^2=o(n)$, which
contains the headline setting $d=\Theta(\log n)$ used here,
$\Pr_\pi(G(\pi)\text{ is simple})=\exp\!\bigl(-(d^2-1)/4 + o(1)\bigr)$
as $n\to\infty$, so
$\Pr_\pi(G(\pi)\text{ is simple})\geq\tfrac{1}{2}e^{-(d^2-1)/4}$ for $n$
sufficiently large. Therefore
$\Pr_{G\sim\mathrm{simple}}[A]\leq 2e^{(d^2-1)/4}\Pr_\pi[A]$.
Combining with Step~2 and inverting for $\delta$ yields
\eqref{eq:probabilistic-distortion}.
\end{proof}

% ============================================================
\section{Hyperparameters}
\label{app:hyperparams}

Table~\ref{tab:hyperparams} lists the Independent PPO hyperparameters
used for all experiments in Section~\ref{sec:experiments}. All values
are held fixed across team sizes $n \in \{4, 8, 16, 100\}$ except as
noted below.

\begin{table}[ht]
\centering
\begin{tabular}{ll}
\toprule
\textbf{Hyperparameter} & \textbf{Value} \\
\midrule
Optimizer & Adam \\
Learning rate & $3 \times 10^{-4}$ \\
Clip parameter $\epsilon$ & $0.2$ \\
Entropy coefficient & $0.01$ \\
Value loss coefficient & $0.5$ \\
Max gradient norm & $0.5$ \\
GAE $\lambda$ & $0.95$ \\
Discount $\gamma$ & $0.99$ \\
Parallel environments & 32 \\
Rollout horizon & 128 \\
PPO epochs per update & 4 \\
Mini-batch size & 512 \\
Training iterations & 100 \\
Hidden layer width & 64 \\
Hidden layers & 2 \\
Activation & Tanh \\
\bottomrule
\end{tabular}
\caption{Independent PPO hyperparameters used for all experiments.
The rollout batch per iteration is $\text{Rollout horizon} \times
\text{Parallel environments} = 128 \times 32 = 4{,}096$ transitions
per agent.}
\label{tab:hyperparams}
\end{table}
The $n = 100$ scaling run of Section~\ref{sec:experiments:scaling} additionally uses 500 iterations, \texttt{num\_envs} $= 32$ (unchanged), and scenario parameters $\texttt{world\_spawning\_x} = \texttt{world\_spawning\_y} = 4.0$, $\texttt{agent\_radius} = 0.03$ to accommodate the larger team size.

% ============================================================
\section{Empirical verification of core propositions}
\label{app:verify}

This appendix contains the protocols, figures, and diagnostics for
the three structural verifications summarized in
Section~\ref{sec:experiments}: exact recovery on $K_n$
(Proposition~\ref{prop:recovery}), unbiasedness of the
Horvitz-Thompson estimator (Proposition~\ref{prop:unbiased}), and
the Hoeffding/Serfling concentration bounds
(Theorem~\ref{thm:concentration}).

\paragraph{Recovery (Proposition~\ref{prop:recovery}).} We evaluate
$\SND(\mathbf{D})$ and $\SNDG(K_4)$ on the same pairwise-distance
matrix at iter~0 (near-homogeneous policies) and iter~100 (trained
policies). The absolute error is $0.0$ in every cell; on a log axis
the gap is indistinguishable from numerical noise at $10^{-12}$
resolution (Figure~\ref{fig:recovery}). This validates that our
Graph-SND implementation is correctly aligned with $\SND$ as a
reference for the remaining experiments.

\begin{figure}[t]
\centering
\includegraphics[width=\linewidth]{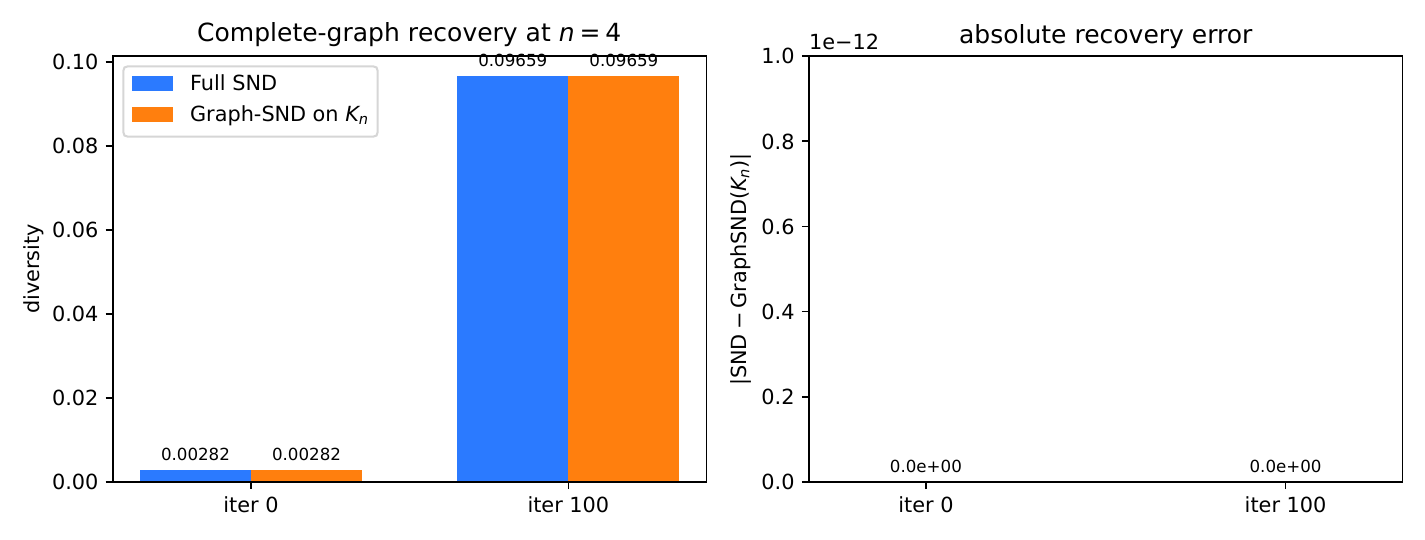}
\caption{Proposition~\ref{prop:recovery}: recovery of $\SND$ by
$\SNDG$ on the complete graph. Left: side-by-side values of $\SND$ and
$\SNDG$ on $K_4$ at iter~0 (near-homogeneous policies) and iter~100
(trained policies). Right: absolute recovery error, which is $0.0$ in
both cases (the y-axis is shown at $10^{-12}$ resolution for context).}
\label{fig:recovery}
\end{figure}

\paragraph{Unbiasedness (Proposition~\ref{prop:unbiased}).} We test
$\E_{G_p}[\widehat{\SND}_{\mathrm{HT}}(G_p)] = \SND(\mathbf{D})$ at $n{=}8$ by
drawing $2{,}000$ independent Bernoulli-$p$ graphs for each
$p \in \{0.1, 0.25, 0.5, 0.75\}$, computing $\widehat{\SND}_{\mathrm{HT}}$ on each,
and reporting the sample mean with a $95\%$ CI
(Figure~\ref{fig:unbiasedness}). The CI contains $\SND(\mathbf{D})$
in every cell across both checkpoints. As a sharper test, the
standardized deviation $|\mathrm{bias}|/\mathrm{SE}$ is approximately
$\mathcal{N}(0,1)$ per cell under an unbiased Gaussian-error
estimator; the maximum observed value across all eight
$(\text{checkpoint}, p)$ cells is $1.42$, well below
$|z_{0.975}| = 1.96$. Estimator variance scales as $1/p$, visible as
widening CIs at small $p$.

\begin{figure}[t]
\centering
\includegraphics[width=\linewidth]{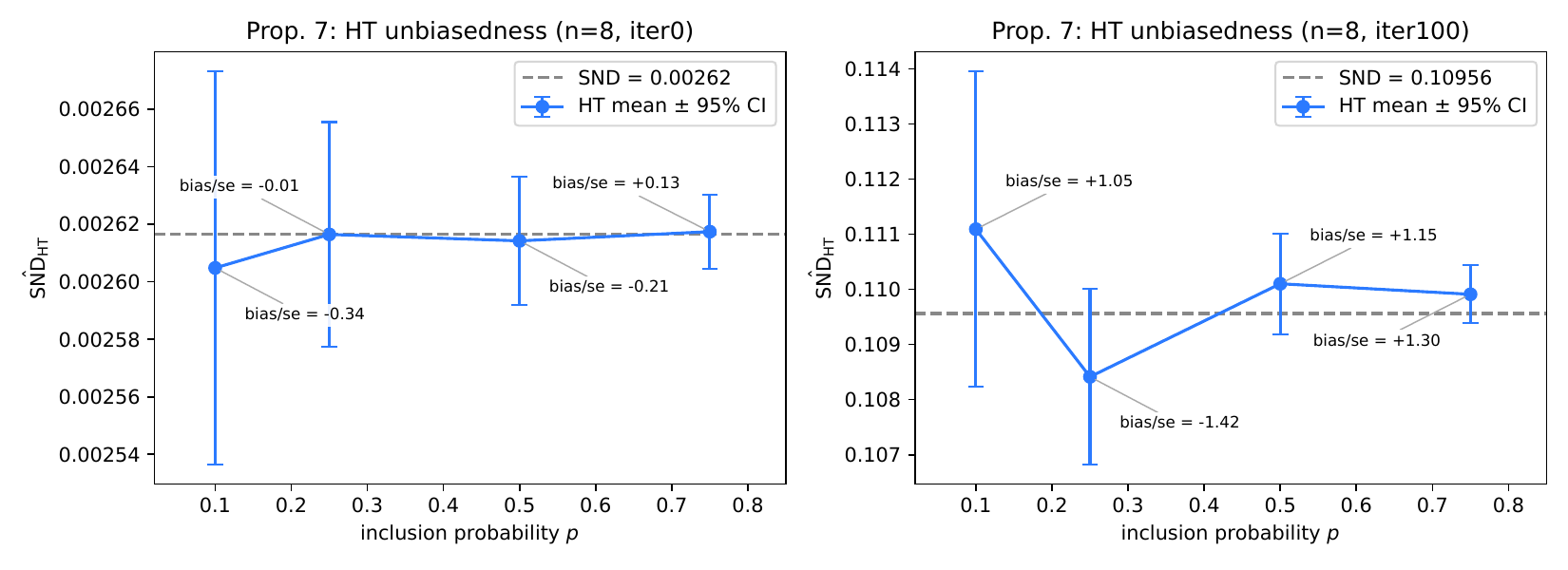}
\caption{Proposition~\ref{prop:unbiased}: empirical unbiasedness of
$\widehat{\SND}_{\mathrm{HT}}(G_p)$ at $n=8$. Blue dots and whiskers are the sample
mean and 95\% confidence interval over 2{,}000 random graphs per $p$;
the dashed line is the true $\SND(\mathbf{D})$. Annotations give the
standardized deviation $|\mathrm{bias}|/\mathrm{SE}$, all of which are
well below $1.96$. Left: iter~0 checkpoint. Right: iter~100 checkpoint.}
\label{fig:unbiasedness}
\end{figure}

\paragraph{Concentration (Theorem~\ref{thm:concentration}).} We
evaluate Theorem~\ref{thm:concentration} at $n{=}16$
($|\Pairs|{=}120$) by drawing $2{,}000$ uniform-size edge samples of
each size $m \in \{6, 12, 24, 48, 72, 96\}$ ($5\%$ to $80\%$ of all
pairs), computing $\SNDG^{\mathrm{u}}$ on each, and comparing the
empirical distribution of $|\SNDG^{\mathrm{u}} - \SND|$ against theory
at $\delta{=}0.10$ (Figure~\ref{fig:concentration}). Across all $12$
cells, the Hoeffding violation rate is $0.00$: no draw out of
$2{,}000$ exceeds $t_H$, well under the allowed $\delta$; the same
holds for the Serfling refinement of Remark~\ref{rem:serfling}. The
empirical curve sits well below both bounds, as expected for
distribution-free guarantees that are worst-case tight only at
population variance $D_{\max}^2/4$; it still exhibits the predicted
$O(1/\sqrt{m})$ slope on the log-log axis, and Serfling is uniformly
tighter than Hoeffding with the gap widening at large sampling
fractions.

\begin{figure}[t]
\centering
\includegraphics[width=\linewidth]{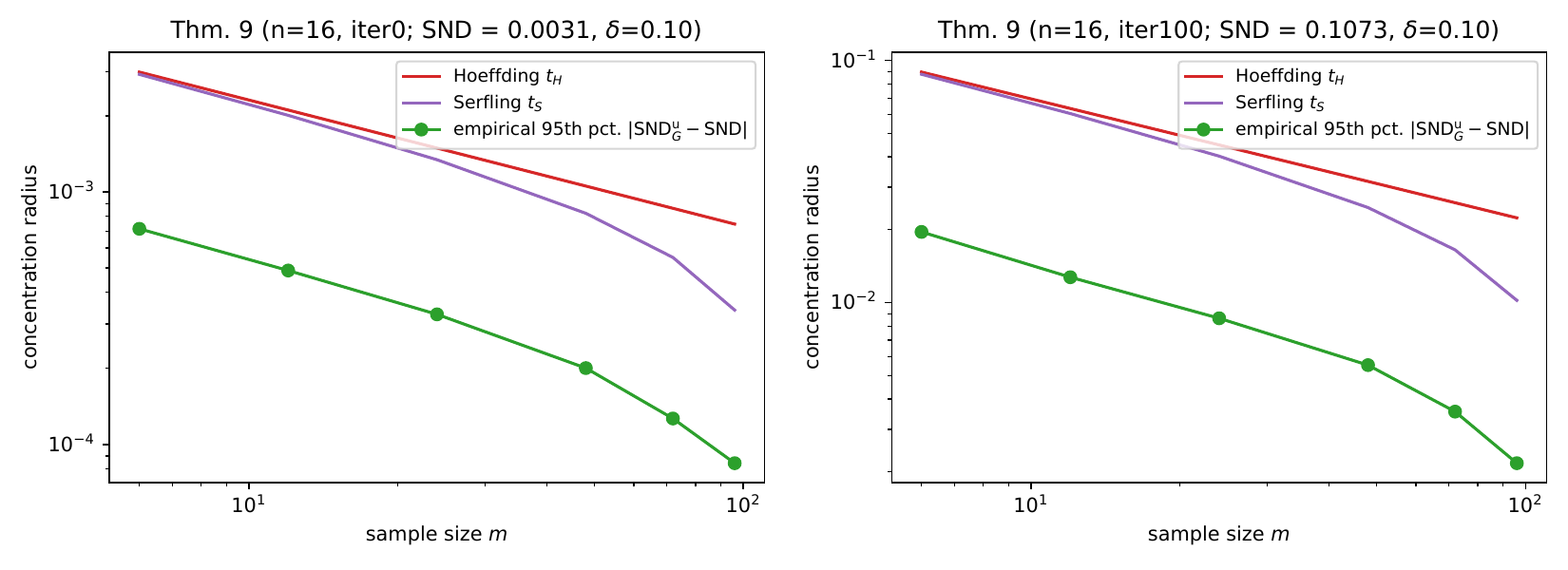}
\caption{Theorem~\ref{thm:concentration}: concentration of
$\SNDG^{\mathrm{u}}$ around $\SND$ at $n=16$. Red: Hoeffding radius
$t_H$ from Theorem~\ref{thm:concentration}. Purple: Serfling radius
$t_S$ from Remark~\ref{rem:serfling}. Green: empirical 95th-percentile
of $|\SNDG^{\mathrm{u}} - \SND|$ over 2{,}000 uniform-size edge
samples per $m$. Both axes are log-scaled. Across all 12 cells, zero
draws violate either bound. The empirical gap reflects the fact that
Hoeffding and Serfling are worst-case over populations bounded in
$[0, D_{\max}]$; the true $d(i,j)$ variance is much smaller.}
\label{fig:concentration}
\end{figure}

\paragraph{Non-VMAS discrete-action transfer (MPE simple-spread + TVD).}
To verify transfer beyond the Gaussian/VMAS regime, we run a PettingZoo
MPE simple-spread panel at $n{=}10$ with discrete actions, categorical
policies, and TVD distance. We train three IPPO seeds with
\texttt{experiments/mpe\_ippo\_training.py}
($500$ iterations, $20$ rollouts/iteration), then run
\texttt{experiments/mpe\_measurement\_panel.py} with $2{,}000$ draws per
seed for three sparse graphs (Bernoulli-$0.1$, Bernoulli-$0.25$,
expander-$d2$) and full SND. We report means across seeds, plus the max
absolute bias over the three seeds. Table~\ref{tab:mpe-tvd-panel} shows
that sparse estimators remain close to full SND (max
$|\mathrm{bias}|=2.92\times10^{-3}$) while reducing metric wall-clock
cost by $1.97$ to $4.87\times$. A random-init baseline panel (same seeds,
same draw budget, no checkpoint load) shows near-zero violation rates
and similarly small bias. This is a transfer/robustness check for the
graph aggregation layer, not a claim of task-level MPE training
superiority.

\begin{table}[t]
\centering
\small
\begin{tabular}{llcccc}
\toprule
\textbf{Regime} & \textbf{Estimator} & \textbf{Mean viol.} & \textbf{Max $|\mathrm{bias}|$} & \textbf{Mean ms} & \textbf{Speedup} \\
\midrule
trained & Bernoulli-$0.1$ & 0.6252 & $2.917{\times}10^{-3}$ & 0.063 & $4.87\times$ \\
trained & Bernoulli-$0.25$ & 0.3757 & $1.306{\times}10^{-3}$ & 0.107 & $2.86\times$ \\
trained & Expander-$d2$ & 0.2302 & $0.703{\times}10^{-3}$ & 0.158 & $1.97\times$ \\
\midrule
random-init & Bernoulli-$0.1$ & 0.0043 & $0.326{\times}10^{-3}$ & 0.063 & $4.86\times$ \\
random-init & Bernoulli-$0.25$ & 0.0000 & $0.184{\times}10^{-3}$ & 0.106 & $2.86\times$ \\
random-init & Expander-$d2$ & 0.0000 & $0.274{\times}10^{-3}$ & 0.143 & $2.13\times$ \\
\bottomrule
\end{tabular}
\caption{PettingZoo MPE simple-spread transfer panel ($n{=}10$,
categorical policies, TVD distance, three seeds, $2{,}000$ draws/seed).
`Mean viol.' is the fraction of draws with
$|\SNDG^{\mathrm{u}}-\SND|>0.05$ averaged across seeds, a
fixed-threshold diagnostic, not the Hoeffding violation rate (which
requires the population-dependent bound from
Theorem~\ref{thm:concentration}). Because TVD lies in $[0,1]$, the
$0.05$ threshold is loose relative to typical Hoeffding radii at
$\delta{=}0.1$; the small max $|\mathrm{bias}|$ column is the more
informative estimator-quality metric. `Speedup' is relative to full SND
wall-clock in the same regime.}
\label{tab:mpe-tvd-panel}
\end{table}

\paragraph{CPU wall-clock sweep at $n\in\{4, 8, 16\}$.} As a
pre-condition to the larger-$n$ GPU results reported in
Section~\ref{sec:experiments:scaling}, we run the same
$\SNDG$-vs-$\SND$ timing comparison on the iter-$100$ checkpoint
across $n \in \{4, 8, 16\}$ and $p \in \{0.1, 0.25, 0.5, 0.75, 1.0\}$,
$20$ trials per cell, CPU only. Figure~\ref{fig:timing} reports
speedup and absolute times: at $p{=}0.1$ we observe
$10$ to $17\times$ speedups across $n$, decaying monotonically to
$\approx 1\times$ at $p{=}1.0$, consistent with the
$\binom{n}{2}/\E[|E(G_p)|]=1/p$ analytical prediction modulo
small-$n$ vectorization overhead.

\begin{figure}[t]
\centering
\includegraphics[width=\linewidth]{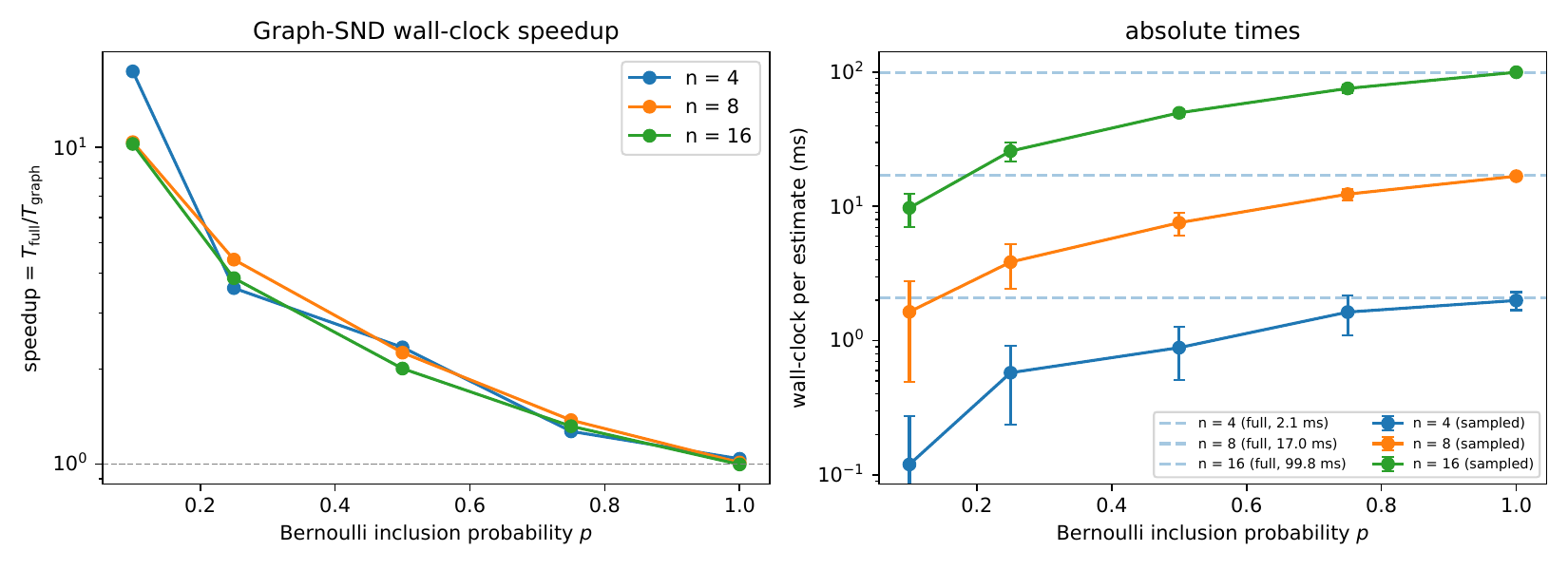}
\caption{Proposition~\ref{prop:complexity}: wall-clock scaling of
$\SNDG$ versus $\SND$ (CPU, $n\in\{4,8,16\}$).
Left: speedup $T_{\mathrm{full}} / T_{\mathrm{sampled}}$ against
Bernoulli inclusion probability $p$, one line per team size $n$.
Right: absolute wall-clock time per estimate with error bars over 20
trials. Dashed horizontal lines mark the full-$\SND$ cost at each
$n$; solid lines are the sampled $\SNDG$ cost. The red curve in the
right panel of Figure~\ref{fig:timing-n500} plots the $p=0.1$
speedups reported here.}
\label{fig:timing}
\end{figure}

\paragraph{Graph construction microbenchmark.}
Because Proposition~\ref{prop:complexity} counts pairwise distance
evaluations and excludes graph construction, we separately time the
CPU-side graph builders used by the experiments. Table~\ref{tab:graph-build}
reports median wall-clock over $50$ trials for $n\in\{50,100\}$ and
$20$ trials for $n=500$ at the same edge budgets used in the expander
ablation. Even exact $k$-NN construction remains below $5$ ms at
$n=500$ in this implementation, far below the full pairwise metric time
reported in Section~\ref{sec:experiments:scaling}; Bernoulli timings in
Figure~\ref{fig:timing-n500} already include sampling and edge-transfer
overhead.

\begin{table}[t]
\centering
\small
\begin{tabular}{lccc}
\toprule
\textbf{Graph builder} & \textbf{$n=50$} & \textbf{$n=100$} & \textbf{$n=500$} \\
\midrule
Bernoulli & $0.016$ ms & $0.036$ ms & $0.875$ ms \\
Uniform-size sample & $0.014$ ms & $0.022$ ms & $0.440$ ms \\
Random $d$-regular & $0.159$ ms & $0.345$ ms & $2.352$ ms \\
Exact $k$-NN & $0.370$ ms & $0.734$ ms & $4.091$ ms \\
\bottomrule
\end{tabular}
\caption{Median CPU graph-construction time at matched edge budgets
$|E|\approx nd/2$, with $d\in\{6,7,9\}$ for
$n\in\{50,100,500\}$. These timings isolate graph construction only;
metric-call timings are reported separately in
Figure~\ref{fig:timing-n500}.}
\label{tab:graph-build}
\end{table}

\section{DiCo head-to-head at $n{=}50$: Bernoulli-$0.1$ Graph-SND vs full SND}
\label{app:dico-n50-sweep}

We extend the closed-loop experiment of
Section~\ref{sec:experiments:dico-dispersion} along three axes: team
size (from $n{=}10$ to $n{=}50$, ${\approx}5\times$ more agents), the
diversity set point (from the single $\SND_{\mathrm{des}}{=}0.1$ used
in the main text to three values
$\SND_{\mathrm{des}}\in\{0.12,0.14,0.15\}$), and \emph{the choice of
estimator itself}: at every $(\mathrm{seed},\SND_{\mathrm{des}})$
cell we run the same $167$-iteration IPPO training twice, once with
Bernoulli-$0.1$ Graph-SND and once with the full SND estimator
(i.e.\ the unmodified DiCo codepath on $K_n$), yielding an
$18$-run ($2{\times}3{\times}3$) head-to-head at matched seeds and
matched hyperparameters. Hyperparameters match
Section~\ref{sec:experiments:dico-dispersion} apart from
$\texttt{on\_policy\_n\_envs\_per\_worker}{=}120$ and
$\texttt{on\_policy\_collected\_frames\_per\_batch}{=}12{,}000$
(reduced from $600$ and $60{,}000$ to keep the per-iteration rollout
within a single RTX~4090's memory budget at $5{\times}$ more
per-agent tensors); the two estimators see bit-identical rollouts
and differ only in the diversity aggregation on each PPO forward.
The expected edge count fed to the controller under the Bernoulli
variant is $\mathbb{E}[|E(G_{0.1})|]{=}0.1\binom{50}{2}{=}122.5$
against $\binom{50}{2}{=}1{,}225$ under full SND ($10\%$ edge
density), a nominal $10\times$ reduction in per-call metric work.

Figure~\ref{fig:dico-n50-bern-vs-full} summarizes the nine paired
runs along four axes. Panel~(a) overlays the per-iteration applied
$\SND{=}\SND_t\,s_t$ trajectories for the two estimators at each set
point; after a ${\sim}25$-iteration transient every cell settles
onto its target line and both estimators produce visually
overlapping bands for the remaining ${\sim}140$ iterations of
training. Panel~(b) shows the corresponding per-iteration episode
reward: the two estimators produce reward curves within observed
seed-to-seed variation across all three set points, with higher set points
yielding monotonically higher terminal reward (consistent with
Dispersion rewarding role specialization among a large team).
Panel~(c) quantifies steady-state tracking: the late-window
($50$-iteration tail) relative error
$|\mathrm{applied}\,\SND - \SND_{\mathrm{des}}|/\SND_{\mathrm{des}}$
is $0.42$ to $0.53\%$ for Bernoulli-$0.1$ and $0.31$ to $0.37\%$ for
full SND, with full SND slightly tighter than Bernoulli-$0.1$ in
every cell; this is expected, since full SND is the ground truth and
Bernoulli-$0.1$ queries only ${\sim}10\%$ of pairs. Both estimators
remain well under the $1\%$ floor of the controller. Panel~(d)
isolates per-call metric cost: Bernoulli-$0.1$ averages
${\sim}2.0\,\mathrm{ms}$/call, full SND averages
${\sim}18.9\,\mathrm{ms}$/call, a $9.24$ to $9.63\times$ ratio that
tracks the analytical $\binom{n}{2}/\mathbb{E}[|E|]{=}10$ prediction
of Proposition~\ref{prop:complexity} nearly exactly. Because full
SND iters take $52.4\,\mathrm{s}$ end-to-end at this $n$ and batch
size, dominated by rollout collection and PPO updates rather than
by the metric; the absolute wall-clock savings at $n{=}50$ are a
few percent of training time; the same ${\sim}10{\times}$
metric-call speedup translates into larger end-to-end savings as the
metric's share of per-iteration time grows with $n$ (see
Figure~\ref{fig:timing-n500}).

Table~\ref{tab:dico-n50-bern-vs-full} reports the numerical summary
aggregated across three seeds per cell. Applied $\SND$ is within
$0.6\%$ of the target for both estimators at all three set points;
cross-seed standard error on applied $\SND$ is at most
$1.2\times 10^{-4}$ (Bernoulli-$0.1$) and $1.0\times 10^{-4}$
(full). Across the nine matched $(\mathrm{seed},\SND_{\mathrm{des}})$
cells, the paired mean reward difference (Bernoulli minus full) is
$-6.2{\times}10^{-4}\pm9.2{\times}10^{-3}$ SEM, with a two-sided sign
test $p{=}1.0$. The head-to-head
thus upgrades the closed-loop conclusion of
Section~\ref{sec:experiments:dico-dispersion} from ``matches full
$\SND$ at $n{=}10$'' to ``matches full $\SND$ at $n{=}50$ over three
distinct set points, with full SND actually trained to convergence
on each cell'', and does so at the $9{-}10\times$ lower per-call
metric cost predicted by the cost invariant. In other words, at
$n{=}50$ the central substitution claim is demonstrated directly
rather than extrapolated: DiCo with Bernoulli-$0.1$ Graph-SND is a
drop-in replacement for DiCo with full SND.

\begin{figure}[t]
\centering
\includegraphics[width=\linewidth]{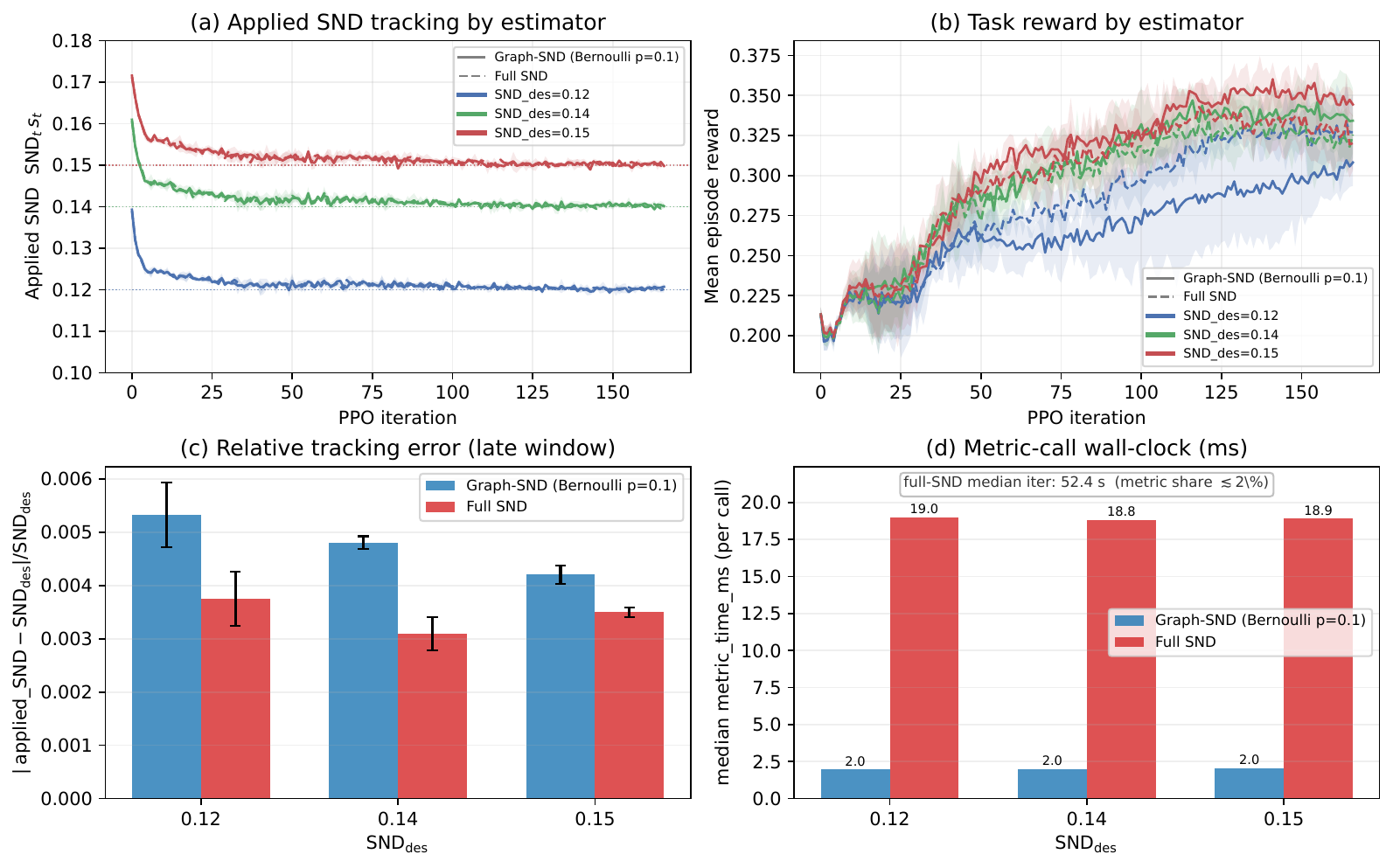}
\caption{DiCo head-to-head at $n{=}50$, $3{\times}3$ grid of seeds
and set points, each cell trained twice with matched rollouts (one
Bernoulli-$0.1$ Graph-SND, one full SND), $18$ runs total.
\textbf{(a)}~Applied $\SND{=}\SND_t\,s_t$ per iteration; color
encodes $\SND_{\mathrm{des}}\in\{0.12,0.14,0.15\}$, solid lines are
Bernoulli-$0.1$, dashed lines are full SND; bands are $\pm 1$ std
across seeds. \textbf{(b)}~Mean episode reward on the same runs,
same color/line encoding. \textbf{(c)}~Late-window ($50$-iter tail)
relative tracking error; full SND is slightly tighter than
Bernoulli-$0.1$ but both remain well under $1\%$.
\textbf{(d)}~Median per-call metric wall-clock; full SND median
iter-time of $52.4\,\mathrm{s}$ is annotated to contextualise the
${\sim}9.5\times$ metric speed-up (which is ${\lesssim}2\%$ of
per-iter time at $n{=}50$ but grows with $n$).}
\label{fig:dico-n50-bern-vs-full}
\end{figure}

\begin{table}[ht]
\centering
\caption{DiCo at $n{=}50$ head-to-head, aggregated across three
seeds on the late window (last $50$ iterations). ``Applied $\SND$''
and ``reward'' are mean $\pm$ cross-seed standard error (SEM);
``rel.\ track.\ err.'' is mean absolute tracking error divided by
$\SND_{\mathrm{des}}$; ``metric time'' is the cross-cell median of
per-run median per-call times for the named estimator.
Bernoulli-$0.1$ has expected edge count $122.5$ out of
$\binom{50}{2}{=}1{,}225$; full SND evaluates all $1{,}225$ pairs.}
\label{tab:dico-n50-bern-vs-full}
\small
\resizebox{\linewidth}{!}{% Auto-generated by experiments/n50_bern_vs_full_comparison.py -- do not edit by hand.
\begin{tabular}{l cc cc cc cc}
\toprule
 & \multicolumn{2}{c}{Applied SND} & \multicolumn{2}{c}{Reward} & \multicolumn{2}{c}{Rel.\ tracking err.} & \multicolumn{2}{c}{Metric time (ms)} \\
$\mathrm{SND}_{\mathrm{des}}$ & Bern-0.1 & Full & Bern-0.1 & Full & Bern-0.1 (\%) & Full (\%) & Bern-0.1 & Full \\
\midrule
0.12 & 0.12032$\pm$0.00012 & 0.12026$\pm$0.00010 & 0.295$\pm$0.016 & 0.3251$\pm$0.0080 & 0.53 & 0.37 & 1.97 & 19.0 \\
0.14 & 0.140354$\pm$0.000073 & 0.140166$\pm$0.000025 & 0.3384$\pm$0.0083 & 0.3268$\pm$0.0072 & 0.48 & 0.31 & 2.00 & 18.8 \\
0.15 & 0.150261$\pm$0.000011 & 0.150301$\pm$0.000015 & 0.3498$\pm$0.0052 & 0.3329$\pm$0.0086 & 0.42 & 0.35 & 2.04 & 18.9 \\
\bottomrule
\end{tabular}
}
\end{table}

To rule out the possibility that the sparse controller merely tracks
its own sparse feedback signal, we repeat the same $18$-cell grid with
an audit callback that computes full complete-graph SND on the scaled
actions at every iteration. This post-hoc signal is never fed back into
the controller. For the Bernoulli-$0.1$ arm, post-hoc full-SND relative
tracking error is $0.43$ to $0.61\%$ across the three set points; for the
full-SND arm it is $0.36$ to $0.48\%$. The mean post-hoc complete-graph
SND exceeds the online applied signal by only
$2.9{\times}10^{-4}$ to $4.7{\times}10^{-4}$, confirming that the
sparse-controlled policies also match the intended full-SND diversity
level.

\begin{figure}[ht]
\centering
\includegraphics[width=\linewidth]{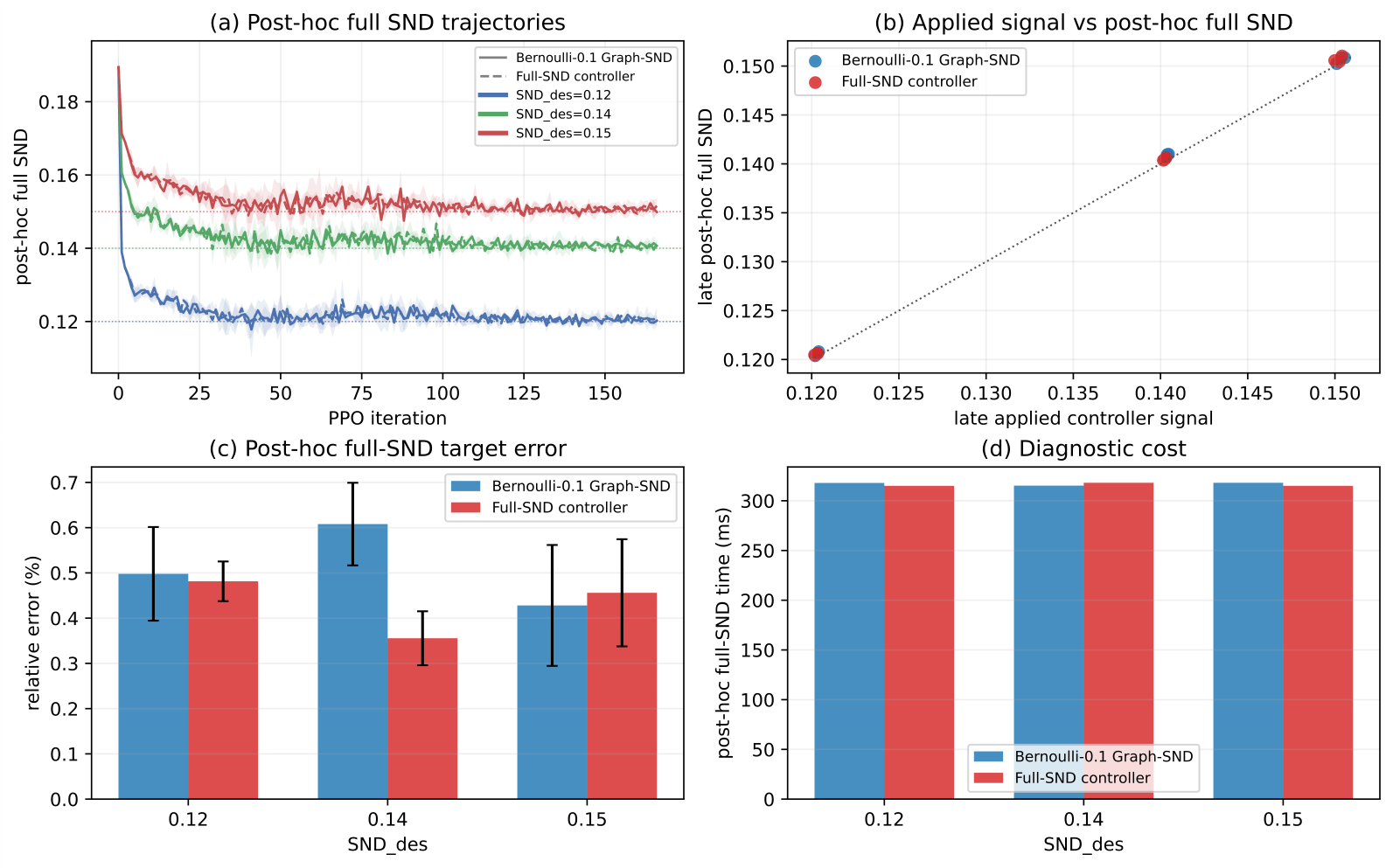}
\caption{Post-hoc complete-graph SND audit for the $n{=}50$ DiCo
head-to-head. The audit computes full SND on the scaled actions every
iteration while leaving the online controller unchanged. Bernoulli-$0.1$
tracks the desired diversity under this complete-graph measurement to
within $0.61\%$ relative error, so the sparse controller is not merely
matching its own sparse feedback signal.}
\label{fig:dico-n50-posthoc-full-snd}
\end{figure}

\section{Expander ablation: supplementary panels}
\label{app:expander-extra}

\begin{figure}[t]
\centering
\includegraphics[width=\textwidth]{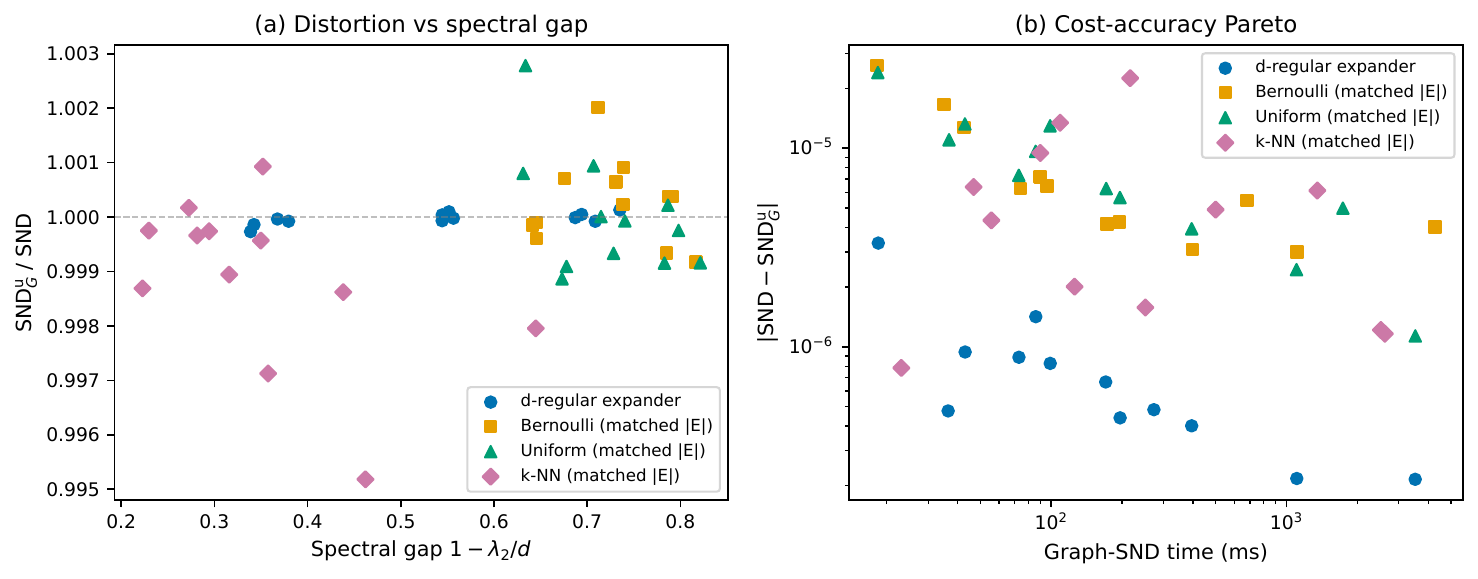}
\caption{Continuation of Figure~\ref{fig:expander-distortion} with additional panels.
\textbf{Expander sparsification ablation (supplementary).}
\textbf{(a)}~Empirical ratio $\SNDG^{\mathrm{u}}/\SND$ vs spectral gap
$1{-}\lambda_2/d$ across four graph families.
\textbf{(b)}~Absolute distortion $|\SND{-}\SNDG^{\mathrm{u}}|$ vs
graph-SND wall-clock time, showing the cost-accuracy Pareto frontier.}
\label{fig:expander-distortion-appendix}
\end{figure}

\begin{figure}[t]
\centering
\includegraphics[width=\textwidth]{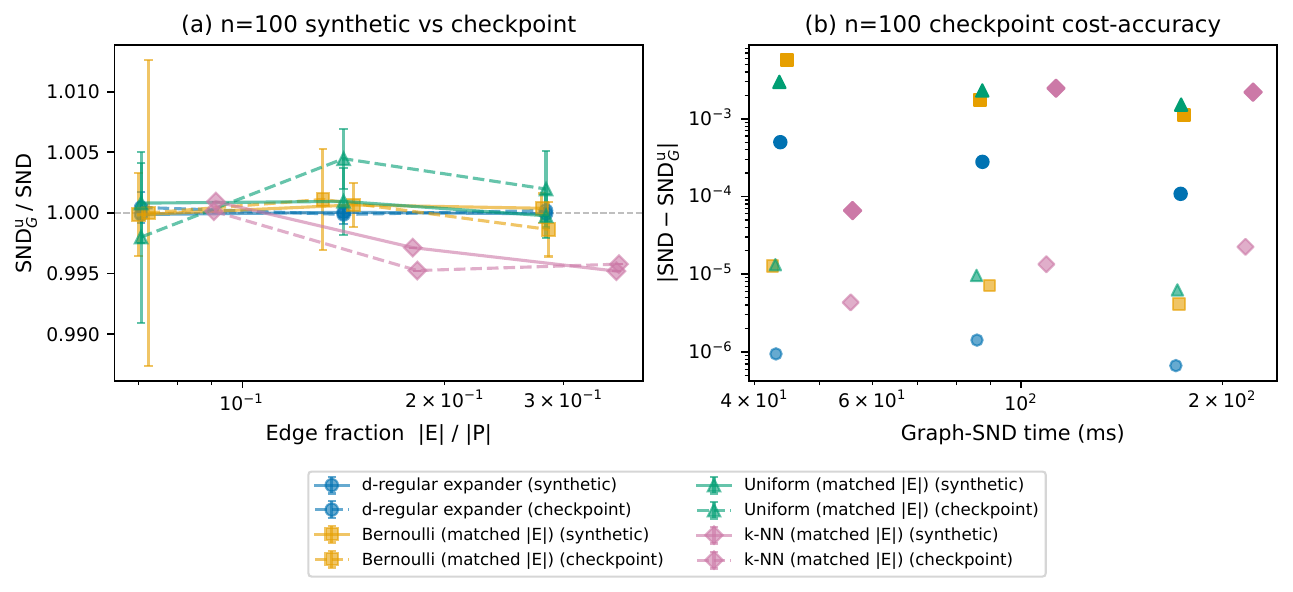}
\caption{Continuation of Figure~\ref{fig:expander-distortion} with additional panels.
\textbf{Expander ablation, trained-policy transfer at $n{=}100$
(supplementary).}
\textbf{(a)}~At matched edge budgets, the distortion ratio
$\SNDG^{\mathrm{u}}/\SND$ for synthetic rollouts (solid) and a trained
checkpoint rollout (dashed) remains near~$1$, with the $d$-regular
family staying most tightly concentrated around unit ratio.
\textbf{(b)}~Cost--accuracy points for synthetic vs trained rollouts:
absolute distortion shifts upward under the larger trained-policy
$\SND$ scale, while the graph-family ordering and timing structure
remain consistent with the structural mechanism of
Theorem~\ref{thm:distortion}.}
\label{fig:expander-distortion-checkpoint}
\end{figure}

% ============================================================
% NeurIPS Paper Checklist (submission/OpenReview only; omit from arXiv).
\ifincludeneuripschecklist
  \clearpage
  \input{checklist}
\fi

\end{document}